\def\year{2021}\relax
\documentclass[letterpaper]{article} 
\usepackage{aaai21} 
\usepackage{times} 
\usepackage{helvet} 
\usepackage{courier} 
\usepackage[hyphens]{url} 
\usepackage{graphicx} 
\usepackage{graphicx} 
\usepackage{natbib} 
\usepackage{caption} 
\usepackage[utf8]{inputenc}
\usepackage{algorithm,algorithmic}
\usepackage{amsmath}
\usepackage{amsthm}
\usepackage{amsfonts}
\usepackage{mathtools}
\usepackage{amssymb}
\def\eqdef{\stackrel{\text{def}}{=}}
\usepackage{bbm}
\usepackage{natbib}

\makeatletter
\newcommand{\printfnsymbol}[1]{%
  \textsuperscript{\@fnsymbol{#1}}%
}

\def\Regret{\mathrm{R}}
\newcommand{\Ocal}{\mathcal{O}}
\newcommand{\Olog}{\tilde{\mathcal{O}}}

\usepackage{xspace}
\DeclareRobustCommand{\eg}{e.g.,\@\xspace}

\newtheorem{proposition}{Proposition}

\newtheorem{lemma}[proposition]{Lemma}
\newtheorem{theorem}[proposition]{Theorem}

\usepackage{mathtools}

\DeclarePairedDelimiter\br{(}{)}
\DeclarePairedDelimiter\brs{[}{]}
\DeclarePairedDelimiter\brc{\{}{\}}
\DeclarePairedDelimiter\abs{\lvert}{\rvert}
\DeclarePairedDelimiter\norm{\lVert}{\rVert}
\DeclarePairedDelimiter\inner{\langle}{\rangle}

\newcommand\Ex[1]{\mathbb{E}\brs*{#1}}

\newcommand\trace[1]{\text{trace}\br{#1}}
\renewcommand{\det}[1]{\text{det}\br{#1}}

\newcommand{\indicator}[1]{\mathbbm{1}{\br*{#1}}}
\newcommand{\E}{\mathbb{E}}
\newcommand{\R}{\mathbb{R}}

\usepackage{thmtools}
\usepackage{thm-restate}

\newtheorem{theorem-rst}[proposition]{Theorem}
\newtheorem{lemma-rst}[proposition]{Lemma}
\newtheorem{proposition-rst}[proposition]{Proposition}
\newtheorem{assumption-rst}[proposition]{Assumption}
\newtheorem{claim-rst}[proposition]{Claim}
\newtheorem{corollary-rst}[proposition]{Corollary}

\usepackage[table, dvipsnames]{xcolor}

\definecolor{darkgreen}{rgb}{0,0.5,0}
\newcommand{\red}[1]{\textcolor{red}{#1}}
\newcommand{\orange}[1]{\textcolor{orange}{#1}}
\newcommand{\green}[1]{\textcolor{darkgreen}{#1}}

\usepackage{selectp}

\newenvironment{proofsketch}{%
  \proof}{\endproof}

\usepackage[switch]{lineno}

\title{Reinforcement Learning with Trajectory Feedback}
\author{Yonathan Efroni\thanks{Equal contribution, alphabetical order}\textsuperscript{\rm 1,2},  Nadav Merlis\printfnsymbol{1}\textsuperscript{\rm 1} and Shie Mannor\textsuperscript{\rm 1,3}\\
}

\affiliations{
    \textsuperscript{\rm 1}Technion -- Israel Institute of Technology, Israel \\
    \textsuperscript{\rm 2}Microsoft Research, New York\\
    \textsuperscript{\rm 3}Nvidia Research, Israel
}

\begin{document}
\maketitle

\begin{abstract}
The standard feedback model of reinforcement learning requires revealing the reward of every visited state-action pair. However, in practice, it is often the case that such frequent feedback is not available. In this work, we take a first step towards relaxing this assumption and require a weaker form of feedback, which we refer to as \emph{trajectory feedback}. Instead of observing the reward obtained after every action, we assume we only receive a score that represents the quality of the whole trajectory observed by the agent, namely, the sum of all rewards obtained over this trajectory. We extend reinforcement learning algorithms to this setting, based on least-squares estimation of the unknown reward, for both the known and unknown transition model cases, and study the performance of these algorithms by analyzing their regret. For cases where the transition model is unknown, we offer a hybrid optimistic-Thompson Sampling approach that results in a tractable algorithm.
\end{abstract}

\section{Introduction}

The field of Reinforcement Learning (RL) tackles the problem of learning how to act optimally in an unknown dynamical environment. Recently, RL witnessed remarkable empirical success (e.g.,~\citealp{mnih2015human,levine2016end,silver2017mastering}). However, there are still some matters that hinder its use in practice. One of them, we claim, is the type of feedback an RL agent is assumed to observe. Specifically, in the standard RL formulation, an agent acts in an unknown environment and receives feedback on its actions in the form of a state-action dependent reward signal. Although such an interaction model seems undemanding at first sight, in many interesting problems, such reward feedback cannot be realized. In practice, and specifically in non-simulated environments, it is hardly ever the case that an agent can query a state-action reward function from every visited state-action pair since such a query can be very costly. For example. consider the following problems:

\vspace{0.1cm}

\emph{(i) Consider the challenge of autonomous car driving. Would we want to deploy an RL algorithm for this setting, we would need a reward signal from every visited state-action pair. Obtaining such data is expected to be very costly since it requires scoring \emph{each} state-action pair with a real number. For example, if a human is involved in providing the feedback, (a) he or she might refuse to supply with such feedback due to the Sisyphean nature of this task, or (b) supplying with such feedback might take too much time for the needs of the algorithm designer.} 

\vspace{0.1cm}

\emph{ (ii) Consider a multi-stage UX interface that we want to optimize using an RL algorithm. To do so, in the standard RL setting, we would need a score for every visited state-action pair. However, as we ask the users for more information on the quality of different state-action pairs, the users' opinions might change due to the information they need to supply. For example, as we ask for more information, the user might be prone to be more negative about the quality of the interface as the user becomes less patient to provide the requested feedback. Thus, we would like to keep the number of queries from the user to be minimal.
}
\vspace{0.1cm}

Rather than circumventing this problem by deploying heuristics (e.g., by hand-engineering a reward signal), in this work, we relax the feedback mechanism to a weaker and more practical one. We then study RL algorithms in the presence of this weaker form of feedback mechanism, a setting which we refer to as \emph{RL with trajectory feedback}. In RL with trajectory feedback, the agent does not have access to a per state-action reward function. Instead, it receives the sum of rewards on the performed trajectory as well as the identity of visited state-action pairs in the trajectory. E.g., for autonomous car driving, we only require feedback on the score of a trajectory, instead of the score of each individual state-action pair. Indeed, this form of feedback is much weaker than the standard RL feedback and is expected to be more common in practical scenarios. 

We start by defining our setting and specifying the interaction model of RL with trajectory feedback (Section~\ref{sec: setup}). In Section~\ref{sec: ls for reward estiamtion}, we introduce a natural least-squares estimator with which the true reward function can be learned based on the trajectory feedback. Building on the least-squares estimator, we study algorithms that explicitly trade-off exploration and exploitation. We start by considering the case where the model is known while the reward function needs to be learned. By generalizing the analysis of standard linear bandit algorithms (OFUL~\citep{abbasi2011improved} and Thompson-Sampling (TS) for linear bandits~\citep{agrawal2013thompson}), we establish performance guarantees for this setup in sections~\ref{sec: oful with known model} and~\ref{sec: ts with known model}. Although the OFUL-based algorithm gives better performance than the TS-based algorithm, its update rule is computationally intractable, as it requires solving a convex maximization problem. Thus, in Section~\ref{sec: ts with unknown model} we generalize the TS-based algorithm to the case where both the reward and the transition model are unknown. To this end, we learn the reward by a TS approach and learn the transition model by an optimistic approach. The combination of the two approaches yields a computationally tractable algorithm, which requires solving an empirical Markov Decision Process (MDP) in each round. For all algorithms, we establish regret guarantees that scale as $\sqrt{K}$, where $K$ is the number of episodes. Finally, in Section~\ref{sec:rarely-switching ucbvi-ts}, we identify the most computationally demanding stage in the algorithm and suggest a variant to the algorithm that rarely performs this stage. Notably, we show that the effect of this modification on the regret is minor. A summary of our results can be found in Table~\ref{tab:results summary}.

\begin{table*}
\begin{center}
\begin{tabular}{|c|c|c|c|c|}\hline \rowcolor{lightgray}
 Result & Exploration& Model Learning & Time Complexity & Regret \\ \hline 
  Theorem~\ref{theorem: OFUL for RL with per trajectory feedback}& OFUL & \red{X}& \red{Computationally-Hard}& {\small$\Olog\br*{SAH \sqrt{ K}}$}\\ 
  \hline
 Theorem~\ref{theorem: performance ts known model}& TS & \red{X}& {\small\orange{$\Ocal\br*{(SA)^3+S^2A(H+A)}$}}& {\small$\Olog\br*{(SA)^{3/2}H\sqrt{K}}$}\\  
 \hline
 Theorem~\ref{theorem: performance ucbvi ts}& TS & \green{V}& {\small\orange{$\Ocal\br*{(SA)^3+S^2A(H+A)}$}} &  {\small$\Olog\br*{S^2A^{3/2}H^{{3/2}}\sqrt{K}}$ }
 \\
 \hline
 Theorem~\ref{theorem: performance switching ucbvi ts}& TS & \green{V}& {\small\green{$\Ocal\br*{S^2A(H+A) + \frac{(SA)^4}{\log(1+C)}\frac{\log\frac{KH}{SA}}{K}}$}}& {\small$\Olog\br*{(SA)^{3/2}H\sqrt{K}\br*{\sqrt{SH} + \sqrt{C}}}$}\\
\hline
\end{tabular}
\end{center}
\caption{$S$ and $A$ are the state and action sizes, respectively, and $H$ is the horizon. $K$ is the number of episode and $C>0$ is a parameter of Algorithm~\ref{alg: ucbvi-ts doubling}. Exploration - whether the reward exploration is optimistic (`OFUL') or uses posterior sampling (`TS'). Model Learning - whether the algorithm knows the model (X) or has to learn it (V). Time complexity - per-episode average time complexity. The hardness of the optimistic algorithm is explained at the end of Section~\ref{sec: oful with known model} and the time complexity of the TS-based algorithm is explained in Section~\ref{sec:rarely-switching ucbvi-ts}. Regret bounds ignore log-factors and constants and assume that $SA\ge H$.}
\label{tab:results summary}
\end{table*}

\section{Notations and Definitions}\label{sec: setup}

We consider finite-horizon MDPs with time-independent dynamics. A finite-horizon MDP is defined by the tuple $\mathcal{M} = \br*{\mathcal{S},\mathcal{A}, R, P, H}$, where $\mathcal{S}$ and $\mathcal{A}$ are the state and action spaces with cardinalities $S$ and $A$, respectively. The immediate reward for taking an action $a$ at state $s$ is a random variable $R(s,a)\in\brs*{0,1}$ with expectation $\Ex{ R(s,a)}=r(s,a)$. The transition kernel is $P(s'\mid s,a)$, the probability of transitioning to state $s'$ upon taking action $a$ at state $s$. $H\in \mathbb{N}$ is the {\em horizon}, i.e.,~the number of time-steps in each episode, and $K\in\mathbb{N}$ is the total number of episodes. We define $[N] \eqdef \brc*{1,\ldots,N},\;$ for all $N\in \mathbb{N}$, and use $h\in\brs*{H}$ and $k\in\brs*{K}$ to denote time-step inside an episode and the index of an episode, respectively. We also denote the initial state in episode $k\in\brs*{K}$ by $s_1^k$, which can be arbitrarily chosen.

A deterministic policy $\pi: \mathcal{S}\times[H]\rightarrow \mathcal{A}$ is a mapping from states and time-step indices to actions. We denote by $a_h \eqdef \pi(s_h,h)$, the action taken at time $h$ at state $s_h$ according to a policy $\pi$. The quality of a policy $\pi$ from state $s$ at time $h$ is measured by its value function, which is defined as
\begin{align*}
    V_h^\pi(s) \eqdef \Ex{\sum_{h'=h}^H r\br*{s_{h'},\pi(s_{h'},h')}\mid s_h=s,\pi},
\end{align*}
where the expectation is over the environment randomness. An optimal policy maximizes this value for all states $s$ and time-steps $h$ simultaneously, and the corresponding optimal value is denoted by $V_h^*(s) \eqdef \max_{\pi} V_h^\pi(s),\;$ for all $h\in [H]$. We can also reformulate the optimization problem using the \emph{occupancy measure}~\citep[\eg][]{puterman1994markov,altman1999constrained}.
The occupancy measure of a policy $\pi$ is defined as the distribution over state-action pairs generated by executing the policy $\pi$ in the finite-horizon MDP $\mathcal{M}$ with a transition kernel $p$~\citep[\eg][]{Zimin2013online}: 
\begin{align*}
    q^\pi_h(s,a;p) &\eqdef \Ex{\indicator{s_h=s,a_h=a}\mid s_1= s_1,p,\pi} \\
    &= \Pr\brc*{s_h=s,a_h=a\mid s_1=s_1,p, \pi}
\end{align*}
For brevity, we define the matrix notation $q^\pi(p)\in \R^{HSA}$ where its $(s,a,h)$ element is given by $q^\pi_h(s,a;p)$. Furthermore, let the average occupancy measure be $d_\pi(p)\in \R^{SA}$ such that $d_\pi(s,a;p)\eqdef \sum_{h=1}^H q_h^\pi(s,a;p).$ For ease of notation, when working with the transition kernel of the true model $p=P$, we write $q^\pi=q^\pi(P)$ and $d_\pi=d_\pi(P)$.

This definition implies the following relation:
\begin{align}\label{eq:prelim_v_to_sa_dist}
    V_1^\pi(s_1;p,r) &=\sum_{h\in[H]} \br*{ \sum_{s_h,a_h} r(s_h, a_h)q^\pi_h(s_h,a_h;p)}\nonumber\\
    &= \sum_{s,a} d_\pi(s,a;p)r(s,a)  = d_\pi(p)^T r,
\end{align}
where $V_1^\pi(s_1;p,r)$ is the value of an MDP whose reward function is $r$ and its transition kernel is $p$.

\paragraph{Interaction Model of Reinforcement Learning with Trajectory Feedback.} We now define the interaction model of RL agents that receive trajectory feedback, the model that we analyze in this work. We consider an agent that repeatedly interacts with an MDP in a sequence of episodes $[K]$. The performance of the agent is measured by its \textit{regret}, defined as $\Regret(K)\eqdef \sum_{k=1}^K \br*{V_1^*(s_1^k) - V_1^{\pi_k}(s_1^k)}$. We denote by $s_h^k$ and $a_h^k$ for the state and the action taken at the $h^{th}$ time-step of the $k^{th}$ episode. At the end of each episode $k\in\brs*{K}$, the agent only observes the cumulative reward experienced while following its policy $\pi_k$ and the identity of the visited state-action pairs, i.e.,
\begin{align}
    \hat{V}_k(s_1^k) = \sum_{h=1}^H R(s_h^k,a_h^k),\ \mathrm{and},\ \brc*{(s_h^k,a_h^k)}_{h=1}^{H+1}. \label{eq: trajectory feedback definition}
\end{align}
This comes in contrast to the standard RL setting, in which the agent observes the reward per visited state-action pair,  $\brc*{R(s_h^k,a_h^k)}_{h=1}^H$. Thus, RL with trajectory feedback receives more obscured feedback from the environment on the quality of its actions. Obviously, standard RL feedback allows calculating $\hat{V}_k(s_1^k)$, but one cannot generally reconstruct $\brc*{R(s_h^k,a_h^k)}_{h=1}^H$ by accessing only $\hat{V}_k(s_1^k)$. 

Next, we define the filtration $F_k$ that includes all events (states, actions, and rewards) until the end of the $k^{th}$ episode, as well as the initial state of the episode $k+1$. We denote by $T=KH$, the total number of time-steps (samples). Moreover, we denote by $n_{k}(s,a)$, the number of times that the agent has visited a state-action pair $(s,a)$, and by $\hat{X}_k$, the empirical average of a random variable $X$. Both quantities are based on experience gathered until the end of the $k^{th}$ episode and are $F_k$ measurable. 

\paragraph{Notations.}
We use $\Olog(X)$ to refer to a quantity that is upper bounded by $X$, up to poly-log factors of $S$, $A$, $T$, $K$, $H$, and $\frac{1}{\delta}$. Furthermore, the notation $\Ocal(X)$ refers to a quantity that is upper bounded by $X$ up to constant multiplicative factors. We use $X \vee Y \eqdef \max\brc*{X,Y}$ and denote $I_{m}$ as the identity matrix in dimension $m$. Similarly, we denote by $\mathbf{0}_m\in\R^m$, the vector whose components are zeros. Finally, for any positive definite matrix $M\in\R^{m\times m}$ and any vector $x\in\R^m$, we define $\norm*{x}_M=\sqrt{x^TMx}$.

\section{From Trajectory Feedback to Least-Squares Estimation}\label{sec: ls for reward estiamtion}

In this section, we examine an intuitive way for estimating the true reward function $r$, given only the cumulative rewards on each of the past trajectories and the identities of visited state-actions. Specifically, we estimate $r$ via a Least-Squares (LS) estimation. Consider past data in the form of~\eqref{eq: trajectory feedback definition}. To make the connection of the trajectory feedback to LS estimation more apparent, let us rewrite~\eqref{eq: trajectory feedback definition} as follows,
\begin{align}
    \hat{V}_k(s_1^k) = \hat{q}_{k}^T R, \label{eq: equivalent form of data}
\end{align}
where $\hat{q}_{k}\in \R^{SAH}$ is the \emph{empirical state-action visitation} vector given by $\hat{q}_{k}(s,a,h) = \indicator{s=s_h^k,a = a_h^k} \in [0,1]$, and $R\in\R^{SAH}$ is the noisy version of the true reward function, namely $R(s,a,h)=R(s_h,a_h)$. Indeed, since the identity of visited state-action pairs is given to us, we can compute $\hat{q}_{k}$ using our data. Furthermore, observe that 
\begin{align*}
    \Ex{\hat{q}_{k}^T R | \hat{q}_k} 
    &= \sum_{s,a,h} \hat{q}_{k}(s,a,h)r(s,a)\\
    &= \sum_{s,a}\br*{\sum_{h=1}^H \hat{q}_{k}(s,a,h)}r(s,a)
    \eqdef\hat{d}_k^T r,
\end{align*}
where the first equality holds since we assume the rewards are i.i.d. and drawn at the beginning of each episode. In the last inequality we defined the \emph{empirical state-action frequency} vector $\hat{d}_k\in \R^{SA}$ where $\hat{d}_k(s,a)=\sum_{h=1}^H \hat{q}_{k}(s,a,h)$, or, alternatively, 
$$\hat{d}_k(s,a)=\sum_{h=1}^H \indicator{s=s_h^k,a = a_h^k} \in [0,H].$$
This observation enables us to think of our data as noisy samples of $\hat{d}_k^T r$, from which it is natural to estimate the reward $r$ by a (regularized) LS estimator, i.e., for some $\lambda>0$,
$$\hat{r}_k\in \arg\min_r \br*{\sum_{l=1}^k (\inner{\hat{d}_l,r} - \hat{V}_{l})^2 + \lambda I_{SA}},$$
This estimator is also given by the closed form solution
\begin{align}
        \hat{r}_k = (D^T_kD_k +\lambda I_{SA})^{-1}Y_k \eqdef A_{k}^{-1} Y_k, \label{eq: ls estimator for r}
\end{align}
where $D_k\in \R^{k\times SA}$ is a matrix with $\brc*{\hat{d}_{k}^T}$ in its rows, $Y_k =\sum_{s=1}^k\hat{d}_s \hat{V}_s\in \R^{SA}$ and $A_{k}= D^T_kD_k +\lambda I_{SA}\in\R^{SA\times SA}$.

A needed property of the estimator $\hat{r}_k$ is for it to be `concentrated' around the true reward $r$. By properly defining the filtration and observing that $\hat{V}_k$ is $\sqrt{H/4}$ sub-Gaussian given $\hat{d}_k$ (as a sum of $H$ independent variables in $\brs*{0,1}$), it is easy to establish a uniform concentration bound via Theorem 2 of~\citep{abbasi2011improved} (for completeness, we provide the proof in Appendix~\ref{appendix:useful results}).
\begin{restatable}[Concentration of Reward]{proposition-rst}{rewardConcentration}\label{proposition: concentration of reward}
Let $\lambda>0$ and $A_{k}\eqdef D^T_kD_k +\lambda I_{SA}$. For any $\delta\in\br*{0,1}$, with probability greater than $1-\delta/10$ uniformly for all $k\ge0$, it holds that 
$\norm{r-\hat{r}_k}_{A_{k}} \leq  \sqrt{\frac{1}{4}SAH\log\br*{\frac{1+kH^2/\lambda}{\delta/10}}}+\sqrt{\lambda SA} \eqdef l_k. $
\end{restatable}

\paragraph{Relation to Linear Bandits.} Assume that the transition kernel $P$ is known while the reward $r$ is unknown. Also, recall that the set of the average occupancy measures, denoted by $\mathcal{K}(P)$, is a convex set~\citep{altman1999constrained}. Then, Equation~\eqref{eq:prelim_v_to_sa_dist} establishes that RL with trajectory feedback can be understood as an instance of linear-bandits over a convex set. I.e., it is equivalent to the problem of minimizing the regret $\Regret(K)=\sum_{k} \max_{d\in \mathcal{K}(P)} d^T r - d_{\pi_k}^Tr$, where the feedback is a noisy version of $d_{\pi_k}^Tr$ since $\Ex{\hat{V}_k\mid F_{k-1}} = d_{\pi_k}^Tr$. Under this formulation, choosing a policy $\pi_k$ is equivalent to choosing an `action' from the convex set $d_{\pi_k}\in\mathcal{K}(P)$.

However, we make use of $\hat{d}_k$, and not the actual `action' that was taken, $d_{\pi_k}$. Importantly, this view allows us to generalize algorithms to the case where the transition model $P$ is unknown (as we do in Section~\ref{sec: ucbvi-ts per trajectory RL}). When the transition model is not known and estimated via $\bar{P}$, there is an \emph{error in identifying the action}, in the view of linear bandits, since $d_\pi \neq d_\pi(\bar P)$. This error, had we used the `na\"ive' linear bandit approach of choosing contexts from $\mathcal{K}(\bar{P})$, would result in errors in the matrix $A_k$. Since our estimator uses the empirical average state-action frequency, $\hat{d}_k$, the fact the model is unknown does not distort the reward estimation. 

\paragraph{Policy Gradient Approach for RL with Trajectory Feedback.} Another natural algorithmic approach to the setting of RL with trajectory feedback is to use policy search. That is, instead of estimating the reward function via least-square and follow a model-based approach, one can directly optimize over the policy class. By the log-derivative trick (as in the REINFORCE algorithm~\citep{williams1992simple}):
{\small
\begin{align*}
    \nabla_\pi &V^\pi(s) \\
    &= \mathbb{E}\brs*{\br*{\sum_{h=1}^H\nabla_\pi \log \pi(a_h|s_h)}\br*{\sum_{h=1}^H r(s_h,a_h)}|s_1=s,\pi}.
\end{align*}}
Thus, if we are supplied with the cumulative reward of a trajectory we can estimate the derivative $\nabla_\pi V^\pi(s)$ and use stochastic gradient ascent algorithm. However, this approach fails in cases the exploration is challenging~\cite{agarwal2020optimality}, i.e., the sample complexity can increase exponentially with $H$. We conjecture that by combining exploration bonus the REINFORCE algorithm can provably perform well, with polynomial sample complexity. We leave such an extension as an interesting future research direction.

\section{OFUL for RL with Trajectory Feedback and Known Model} \label{sec: oful with known model}

\begin{algorithm}[t]
\caption{OFUL for RL with Trajectory Feedback and Known Model} \label{alg: oful known model}
\begin{algorithmic}
\STATE {\bf Require:} $\delta\in(0,1)$, $\lambda=H$,\\ $\qquad\qquad l_k= \sqrt{\frac{1}{4}SAH\log\br*{\frac{1+kH^2/\lambda}{\delta/10}}}+\sqrt{\lambda SA}$
\STATE {\bf Initialize:} $A_0=\lambda I_{SA}$, $Y_0=\mathbf{0}_{SA}$ 
\FOR{$k=1,...,K$}
\STATE Calculate $\hat{r}_{k-1}$ via LS estimation~\eqref{eq: ls estimator for r}
\STATE Solve $\pi_k\in \arg\max_{\pi}\br*{ d_\pi^T\hat{r}_{k-1}+ l_{k-1}\norm{d_{\pi}}_{A_{k-1}^{-1}}}$ 
\STATE Play $\pi_k$, observe $\hat{V}_k$ and $\brc*{(s_h^k,a_h^k)}_{h=1}^H$
\STATE Update $A_{k}= A_{k-1} + \hat{d}_k \hat{d}_k^T $ and $Y_{k} = Y_{k-1} + \hat{d}_k \hat{V}_k$.
\ENDFOR
\end{algorithmic}
\end{algorithm}

Given the concentration of the estimated reward in Proposition~\ref{proposition: concentration of reward}, it is natural to follow the optimism in the face of uncertainty approach, as used in the OFUL algorithm~\citep{abbasi2011improved} for linear bandits. We adapt this approach to RL with trajectory feedback, as depicted in Algorithm~\ref{alg: oful known model}; on each episode, we find a policy that maximizes the estimated value $V^\pi(s_1; P, \hat{r}_{k-1})=d_\pi^T\hat{r}_{k-1}$, and an additional `confidence' term $l_{k-1}\norm{d_{\pi}}_{A_{k-1}^{-1}}$ that properly encourages the policy $\pi_k$ to be exploratory.

The analysis of OFUL is based upon two key ingredients, (i) a concentration result, and (ii) an elliptical potential lemma. For the setting of RL with trajectory feedback, the concentration result can be established with similar tools to~\citeauthor{abbasi2011improved} (see Proposition~\ref{proposition: concentration of reward}). However, (ii), the elliptical potential lemma, should be re-derived. The usual elliptical potential lemma~\cite{abbasi2011improved} states that for $x_k\in\R^m$,
{\small
\begin{align*}
    \sum_{k=1}^K \norm{x_k}_{A_{k-1}}\leq \Olog\br*{\norm{x}\sqrt{mK/\lambda}},
\end{align*}}
where $A_{k}=A_{k-1} + x_kx_k^T,A_0=\lambda I_m$ and $\norm{x_k}\leq \norm{x}$. However, for RL with trajectory feedback, the term we wish to bound is $\sum_{k=1}^K \norm{d_{\pi_k}}_{A_{k-1}},$ where $A_{k}=A_{k-1} + \hat{d}_k\hat{d}_k^T$, $A_0=\lambda I_{SA}$. Thus, since $d_{\pi_k}\neq \hat{d}_k$, we cannot apply the usual elliptical potential lemma. Luckily, it is possible to derive a variation of the lemma, suited for our needs, by recognizing that $d_{\pi_k} = \Ex{ \hat{d}_k| F_{k-1}}$ where the equality holds component-wise. Based on this observation, the next lemma -- central to the analysis of all algorithms in this work -- can be established (in Appendix~\ref{supp: expected elliptical potential lemma} we prove a slightly more general statement that will be useful in next sections as well).

\begin{restatable}[Expected Elliptical Potential Lemma]{lemma-rst}{ExpectedPotentialLemmaPaper} \label{lemma: elliptical potential lemma with side information paper}
    Let $\lambda>0$. Then, uniformly for all $K>0$, with probability greater than $1-\delta$, it holds that
    {\small
    \begin{align*}
        &\sum_{k=0}^{K}\norm{d_{\pi_k}}_{A_{k-1}^{-1}} \leq  \Ocal\br*{\sqrt{\frac{H^2}{\lambda}KSA\log\br*{\lambda + \frac{KH^2}{SA}}}}.
    \end{align*}
    }
\end{restatable} 

    \begin{proofsketch}
    
        Applying Jensen's inequality (using the convexity of norms), we get
        {\small
       \begin{align*}
            &\norm{d_{\pi_k}}_{A_{k-1}^{-1}} = \norm*{\Ex{\hat{d}_{k} | F_{k-1}}}_{A_{k-1}^{-1}}
            \leq \Ex{\norm{\hat{d}_{k} }_{A_{k-1}^{-1}}\vert F_{k-1}}
        \end{align*}
        }
       Therefore, we can write
       {\small
        \begin{align*}
            &\sum_{k=0}^{K} \norm{d_{\pi_k}}_{A_{k-1}^{-1}} \leq \sum_{k=1}^{K}\Ex{ \norm{\hat{d}_{k} }_{A_{k-1}^{-1}}|  | F_{k-1}} \\
            &=\underbrace{\sum_{k=1}^{K}\br*{\Ex{ \norm{\hat{d}_{k} }_{A_{k-1}^{-1}}|  | F_{k-1}} - \norm{\hat{d}_{k} }_{A_{k-1}^{-1}}}}_{(a)} + \underbrace{\sum_{k=1}^{K} \norm{\hat{d}_{k} }_{A_{k-1}^{-1}}}_{(b)}\;,
        \end{align*}
        }
        where in the last relation, we added and subtracted the random variable $\norm{\hat{d}_{k} }_{A_{k-1}^{-1}}$. It is evident that $(a)$ is a bounded martingale difference sequence and, thus, can be bounded with probability $1-\delta/2$ by
        {\small
        \begin{align*}
        (a) \leq   4\sqrt{\frac{H^2}{\lambda}K\log\br*{\frac{2K}{\delta}}}.
        \end{align*}
        }
        Term $(b)$ can be bounded by applying the usual elliptical potential \cite{abbasi2011improved}) by
        {\small
        \begin{align*}
            (b) \leq \sqrt{\frac{H^2}{\lambda}} \sqrt{2KSA\log\br*{\lambda + \frac{KH^2}{SA}}}.
        \end{align*}
        }
        Combining the bounds on $(a)$, $(b)$ concludes the proof.
    \end{proofsketch}

Based on the concentration of the estimated reward $\hat{r}_k$ around the true reward $r$ (Proposition~\ref{proposition: concentration of reward}) and the expected elliptical potential lemma (Lemma~\ref{lemma: elliptical potential lemma with side information paper}), the following performance guarantee of Algorithm~\ref{alg: oful known model} is established (see Appendix~\ref{supp: OFUL RL per trajectory feedback} for the full proof).

\begin{restatable}[OFUL for RL with Trajectory Feedback and Known Model]{theorem-rst}{OFULrlPer}\label{theorem: OFUL for RL with per trajectory feedback}
For any $\delta\in (0,1)$, it holds with probability greater than $1-\delta$ that for all $K>0$,
{\small
$$
\Regret(K)\leq \Ocal\br*{SAH \sqrt{ K} \log\br*{\frac{KH}{\delta}}}.
$$
}
\end{restatable}

To exemplify how the expected elliptical potential lemma is applied in the analysis of Algorithm~\ref{alg: oful known model} we supply a sketch of the proof.
\begin{proofsketch}
By the optimism of the update rule, following~\citep{abbasi2011improved}, it is possible to show that with high probability,
$$
V^*_1(s_1^k) = d_{\pi^*}^T r \leq d_{\pi_k}^T\hat{r}_{k-1}+ l_{k-1}\norm{d_{\pi_k}}_{A_{k-1}^{-1}}, 
$$
for any $k>0$. Thus, we only need to bound the \emph{on-policy} prediction error given as follows,
{\small
\begin{align}
    \Regret(K)
    &= \sum_{k=1}^K (V_1^* - V_1^{\pi_k}) \nonumber\\
    &\leq \sum_{k=1}^K (d_{\pi_k}^T\hat{r}_{k-1} + l_{k-1} \norm{d_{\pi_k}}_{A_{k-1}^{-1}}  - d_{\pi_k}^T r)\nonumber\\
    &\leq 2l_K \sum_{k=1}^K \norm{d_{\pi_k}}_{A_{k-1}^{-1}}\,. \label{eq: sketch oful rel 1}
\end{align}
}
where the last inequality can be derived using Proposition~\ref{proposition: concentration of reward} and the Cauchy Schwartz inequality. Applying the expected elliptical potential lemma (Lemma~\ref{lemma: elliptical potential lemma with side information paper}) and setting $\lambda=H$ concludes the proof.
\end{proofsketch}

Although Algorithm~\ref{alg: oful known model} provides a natural solution to the problem, it results in a major computational disadvantage. The optimization problem needed to be solved in each iteration is a convex maximization problem (known to generally be NP-hard \citep{atamturk2017maximizing}). Furthermore, since $\norm{d_{\pi}}_{A_{k-1}^{-1}}$ is non-linear in $d_\pi$, it restricts us from solving this problem by means of Dynamic Programming. In the next section, we follow a different route and formulate a \emph{Thompson Sampling} based algorithm, with computational complexity that amounts to sampling a Gaussian noise for the reward and solving an MDP at each episode.

\begin{algorithm}[t]
\caption{TS for RL with Trajectory Feedback and Known Model} \label{alg: ts known model}
\begin{algorithmic}
\STATE {\bf Require:} $\delta\in(0,1)$, $\lambda=H$, $v_k=\sqrt{9SAH\log\frac{kH^2}{\delta/10}}$\\ $\qquad\qquad l_k= \sqrt{\frac{1}{4}SAH\log\br*{\frac{1+kH^2/\lambda}{\delta/10}}}+\sqrt{\lambda SA}$ 
\STATE {\bf Initialize:} $A_0=\lambda I_{SA}$, $Y_0=\mathbf{0}_{SA}$
\FOR{$k=1,...,K$}
\STATE Calculate $\hat{r}_{k-1}$ via LS estimation~\eqref{eq: ls estimator for r}
\STATE Draw noise {\small$\xi_k\!\sim\! \mathcal{N}(0,v_{k}^2A_{k-1}^{-1})$} and define {\small$\tilde{r}_k\!=\!\hat{r}_{k-1}+\xi_k$}
\STATE Solve an MDP with perturbed empirical reward $\pi_k\in \arg\max_{\pi} d(P)_\pi^T\tilde{r}_k$ 
\STATE Play $\pi_k$, observe $\hat{V}_k$ and $\brc*{(s_h^k,a_h^k)}_{h=1}^H$
\STATE  Update $A_{k}= A_{k-1} + \hat{d}_k \hat{d}_k^T $ and $Y_{k} = Y_{k-1} + \hat{d}_k \hat{V}_k$.
\ENDFOR
\end{algorithmic}
\end{algorithm}

\section{Thompson Sampling for RL with Trajectory Feedback} \label{sec: ucbvi-ts per trajectory RL}

The OFUL-based algorithm for RL with trajectory feedback, analyzed in the previous section, was shown to give good performance in terms of regret. However, implementing the algorithm requires solving a convex maximization problem before each episode, which is, in general, computationally hard. Instead of following the OFUL-based approach, in this section, we analyze a Thompson Sampling (TS) approach for RL with trajectory feedback.

We start by studying the performance of Algorithm~\ref{alg: ts known model}, which assumes access to the transition model (as in Section~\ref{sec: oful with known model}). Then, we study Algorithm~\ref{alg: ucbvi-ts} which generalizes the latter method to the case where the transition model is unknown. In this generalization, we use an optimistic-based approach to learn the \emph{transition model}, and a TS-based approach to learn the \emph{reward}. The combination of optimism and TS results in a tractable algorithm in which every iteration amounts to solving an empirical MDP (which can be done by Dynamic Programming). The reward estimator in both Algorithm~\ref{alg: ts known model} and Algorithm~\ref{alg: ucbvi-ts} is the same LS estimator~\eqref{eq: ls estimator for r} used for the OFUL-like algorithm. Finally, we focus on improving the most computationally-demanding stage of Algorithm~\ref{alg: ucbvi-ts}, which is the reward sampling, and suggest a more efficient method in Algorithm~\ref{alg: ucbvi-ts doubling}.

\subsection{TS for RL with Trajectory Feedback and Known Model}\label{sec: ts with known model}

For general action sets, it is known that OFUL~\citep{abbasi2011improved}  results in a computationally intractable update rule. One popular approach to mitigate the computational burden is to resort to TS for linear bandits~\citep{agrawal2013thompson}. Then, the update rule amounts to solving a linear optimization problem over the action set. Yet, the reduced computational complexity of TS comes at the cost of an increase in the regret. Specifically, for linear bandit problems in dimension $m$, OFUL achieves $\Olog(m\sqrt{T})$, whereas TS achieves $\Olog(m^{3/2}\sqrt{T})$~\citep{agrawal2013thompson,abeille2017linear}. 

Algorithm~\ref{alg: ts known model} can be understood as a TS variant of Algorithm~\ref{alg: oful known model}, much like TS for linear bandits~\cite{agrawal2013thompson} is a TS variant of OFUL. Unlike the common TS algorithm for linear bandits, Algorithm~\ref{alg: ts known model} uses the LS estimator in Section~\ref{sec: ls for reward estiamtion}, i.e., the one which uses the empirical state-action distributions $\hat{d}_k$, instead of the `true action' $d_{\pi_k}$. In terms of analysis, we deal with this discrepancy by applying -- as in Section~\ref{sec: oful with known model} -- the expected elliptical potential lemma,\footnote{We use a slightly more general version of the expected elliptical potential lemma, presented in Appendix~\ref{supp: expected elliptical potential lemma}} instead of the standard elliptical potential lemma. Then, by extending techniques from~\citep{agrawal2013thompson,russo2019worst} we obtain the following performance guarantee for Algorithm~\ref{alg: ts known model} (see Appendix~\ref{supp: proof of ts with known model} for the full proof).

\begin{restatable}[TS for RL with Trajectory Feedback and Known Model]{theorem-rst}{theoremPerformanceTSUCRL}\label{theorem: performance ts known model}
For any $\delta\in (0,1)$, it holds with probability greater than $1-\delta$ that for all $K>0$,
{\small
\begin{align*}
\Regret(K) 
&\leq \Ocal\br*{(SA)^{3/2}H\sqrt{K \log(K)}\log\br*{\frac{KH}{\delta}}} \\
&\quad+ \Ocal\br*{SAH\sqrt{\log\br*{\frac{KH^2}{\delta}}} } .
\end{align*}
}
\end{restatable}
Observe that Theorem~\ref{theorem: performance ts known model} establishes a regret guarantee of $m^{3/2}\sqrt{K}$ since the dimension of the specific linear bandit problem is $m=SA$ (see~\eqref{eq:prelim_v_to_sa_dist}). This is the type of regret is expected due to TS type of analysis~\citep{agrawal2013thompson}. It is an interesting question whether this bound can be improved due to the structure of the problem. 

\subsection{UCBVI-TS for RL with Trajectory Feedback}\label{sec: ts with unknown model}

In previous sections, we devised algorithms for RL with trajectory feedback, assuming access to the true transition model and that only the reward function is needed to be learned. In this section, we relax this assumption and study the setting in which the transition model is also unknown. 

This setting highlights the importance of the LS estimator~\eqref{eq: ls estimator for r}, which uses the empirical state-action frequency $\hat{d}_k$, instead of $d_{\pi_k}$. I.e., when the transition model is not known, we do not have access to $d_{\pi_k}$. Nevertheless, it is reasonable to assume access to $\hat{d}_k$ since it only depends on the \emph{observed} sequence of state-action pairs in the $k^{th}$ episode $\brc*{s_h^k,a_h^k}_{h=1}^H$ and does not require any access to the true model. For this reason, the LS estimator~\eqref{eq: ls estimator for r} is much more amenable to use in RL with trajectory feedback when the transition model \emph{is not given} and needed to be estimated.

Algorithm~\ref{alg: ucbvi-ts}, which we refer to as UCBVI-TS (Upper Confidence Bound Value Iteration and Thompson Sampling), uses a combined TS and optimistic approach for RL with trajectory feedback. At each episode, the algorithm perturbs the LS estimation of the reward $\hat{r}_{k-1}$ by a random Gaussian noise $\xi_k$, similarly to Algorithm~\ref{alg: ts known model}. Furthermore, to encourage the agent to learn the unknown transition model, UCBVI-TS adds to the reward estimation the bonus $b_{k-1}^{pv}\in \R^{SA}$ where 
\begin{align}
    b_{k-1}^{pv}(s,a) \simeq \frac{H}{\sqrt{n_{k-1}(s,a)\vee 1}}, \label{eq: bonus ucbvi optimistic model}
\end{align}
up to logarithmic factors (similarly to \citealt{azar2017minimax}). Then, it simply solves the empirical MDP defined by the plug-in transition model $\bar{P}_{k-1}$ and the reward function $\hat{r}_{k-1} +\xi_k+b_{k-1}^{pv}$. Specifically, the transition kernel $\bar{P}_{k-1}$ is estimated by
{\small
\begin{align}
    \bar{P}_{k}(s' \!\mid\! s,a) \!=\! \frac{\sum_{l=1}^{k}\sum_{h=1}^{H} \indicator{s_h^l=s,a_h^l=a,s_{h+1}^l=s'}}{n_{k}(s,a)\vee 1}. \label{eq:transition estimation}
\end{align}
}
The next result establishes a performance guarantee for UCBVI-TS with trajectory feedback (see proof in Appendix~\ref{supp: proof of ucbvi-ts}). The key idea for the proof is showing that the additional bonus term~\eqref{eq: bonus ucbvi optimistic model} induces sufficient amount of optimism with fixed probability. Then, by generalizing the analysis of Theorem~\ref{theorem: performance ts known model} while using some structural properties of MDPs we derive the final result.  

\begin{restatable}[UCBVI-TS Performance Guarantee]{theorem-rst}{theoremPerformanceTSUCBVI}\label{theorem: performance ucbvi ts}
For any $\delta\in (0,1)$, it holds with probability greater than $1-\delta$ that for all $K>0$,
{\small
\begin{align*}
    \Regret&(K)\leq 
    \Ocal\br*{SH(SA+H)\sqrt{AHK\log K} \log\br*{\frac{SAHK}{\delta}}^{\frac{3}{2}} }\\
    &\quad+ \Ocal\br*{H^2\sqrt{S}(SA + H)^2\log\br*{\frac{SAHK}{\delta}}^2\sqrt{\log K}}
\end{align*}
}
thus, discarding logarithmic factors and constants and assuming $SA\geq H$, 
$\Regret(K)\leq \Olog\br*{S^2A^{3/2}H^{{3/2}}\sqrt{K}}.$
\end{restatable}

\begin{algorithm}[t]
\caption{UCBVI-TS for RL with Trajectory Feedback} \label{alg: ucbvi-ts}
\begin{algorithmic}
\STATE {\bf Require:} {\small$\delta\in(0,1)$, $\lambda=H$, $v_k=\sqrt{9SAH\log\frac{kH^2}{\delta/10}}$,\\ 
$\qquad\qquad l_k= \sqrt{\frac{1}{4}SAH\log\br*{\frac{1+kH^2/\lambda}{\delta/10}}}+\sqrt{\lambda SA}$,\\ 
$\qquad\qquad b^{pv}_k(s,a) = \sqrt{\frac{H^2\log \frac{40SAH^2k^3}{\delta}}{n_{k}(s,a)\vee 1}}$}
\STATE {\bf Initialize:} $A_0=\lambda I_{SA}$, $Y_0=\mathbf{0}_{SA}$, \\ \qquad\qquad\, Counters $n_{0}(s,a) = 0,\, \forall s,a$.
\FOR{$k=1,...,K$}
\STATE Calculate $\hat{r}_{k-1}$ via LS estimation~\eqref{eq: ls estimator for r} and $\bar{P}_{k}$ by \eqref{eq:transition estimation}
\STATE Draw noise $\xi_k\sim \mathcal{N}(0,v_{k}^2A_{k-1}^{-1})$ and define $\tilde{r}_k^b = \hat{r}_{k-1}+\xi_k+b^{pv}_{k-1}$
\STATE Solve empirical MDP with optimistic-perturbed reward, $\pi_k\in \arg\max_{\pi} d(\bar{P}_{k-1})^T\tilde{r}_k^b$ 
\STATE Play $\pi_k$, observe $\hat{V}_k$ and $\brc*{(s_h^k,a_h^k)}_{h=1}^H$
\STATE  Update counters $n_k(s,a)$, $A_{k}= A_{k-1} + \hat{d}_k \hat{d}_k^T$ and $Y_{k} = Y_{k-1} + \hat{d}_k \hat{V}_k$.
\ENDFOR
\end{algorithmic}
\end{algorithm}

\subsection{Improving the Computational Efficiency of UCBVI-TS}
\label{sec:rarely-switching ucbvi-ts}
In this section, we present a modification to UCBVI-TS that uses a doubling trick to improve the computational efficiency of Algorithm~\ref{alg: ucbvi-ts}. Specifically, for $m=SA$, the complexity of different parts of UCBVI-TS is as follows:
\begin{itemize}
\item $A_{k}^{-1}$ can be iteratively updated using $\Ocal\br*{m^2}$ computations by the Sherman-Morrison formula \citep{bartlett1951inverse}.
\item Given $A_{k}^{-1}$, calculating $\hat{r}_k$ requires $\Ocal\br*{m^2}$ computations.
\item Generating the noise $\xi_k$ requires calculating $A_{k}^{-1/2}$ that, in general, requires performing a singular value decomposition to $A_{k}^{-1}$ at a cost of $\Ocal(m^3)$ computations.
\item Finally, calculating the policy using dynamic programming requires $\Ocal(S^2AH)$ computations.
\end{itemize}
In overall, UCBVI-TS requires $\Ocal\br*{(SA)^3+S^2A(H+A)}$ computations per episode, where the most demanding part is the noise generation. Thus, we suggest a variant of our algorithm, called Rarely-Switching UCBVI-TS (see Appendix~\ref{supp: rarely switching ucbvi-ts} for the full description of the algorithm), that updates $A_k$ (and, as a result, the LS estimator) only after the updates increase the determinant of $A_k$ by a factor of $1+C$, for some $C>0$, similarly to \citep{abbasi2011improved}. Specifically, we let $B_{k}=\lambda I_{SA} + \sum_{s=1}^k\hat{d}_s \hat{d}_s^T$ and update $A_k=B_k$ if $\det{B_k}>(1+C)\det{A_k}$, where $B_0=A_0=\lambda I_{SA}$. Otherwise, we keep $A_k=A_{k-1}$. By the matrix-determinant lemma, $\det{B_k}$ can be iteratively updated by $\det{B_k} = \br*{1+\hat{d}_k^TB_{k-1}^{-1}\hat{d}_k}\det{B_{k-1}}$, which requires $\Ocal(SA)$ calculations given $B_{k-1}^{-1}$; in turn, $B_{k-1}^{-1}$ can be updated by the Sherman-Morrison formula. Notably, $A_k$ is rarely updated, as we prove in the following lemma:

\begin{restatable}{lemma-rst}{switchesBound} \label{lemma: switches bound}
    Under the update rule of Rarely-Switching UCBVI-TS, and for any $C>0,\lambda>0$, it holds that 
    $
        \sum_{k=1}^K \indicator{A_k \neq A_{k-1}} \le \frac{SA}{\log (1+C)}\log\br*{1 + \frac{KH^2}{\lambda SA}}.
    $
     
\end{restatable}
Therefore, the average per-round computational complexity of Rarely-Switching UCBVI-TS after $K$ episodes is
\begin{align*}
    \Ocal\br*{S^2A(H+A) + \frac{(SA)^4}{\log(1+C)}\frac{\log\frac{KH}{SA}}{K}}\enspace.
\end{align*}
Moreover, rarely updating $A_k$ only affects the lower-order terms of the regret, as we prove in the following theorem: 
\begin{restatable}[Rarely-Switching UCBVI-TS Performance Guarantee]{theorem-rst}{theoremSwitchingTSUCBVI}\label{theorem: performance switching ucbvi ts}
For any $\delta\in (0,1)$, it holds with probability greater than $1-\delta$ that for all $K>0$,
{\small
\begin{align*}
    \Regret(K)
    &\leq \Olog\br*{SH(SA+H)\sqrt{AHK} + H^2\sqrt{S}(SA + H)^2} \\
    &\quad+ \Olog\br*{(SA)^{3/2}H\sqrt{(1+C)K}}\enspace.
\end{align*}
}
\end{restatable}
The proof can be found in Appendix~\ref{supp: rarely switching ucbvi-ts}. See that the difference from Theorem~\ref{theorem: performance ucbvi ts} is in the last term, which is negligible, compared to the first term, for reasonable values of $C>0$.

\section{Discussion and Conclusions}

In this work, we formulated the framework of RL with trajectory feedback and studied different RL algorithms in the presence of such feedback. Indeed, in practical scenarios, such feedback is more reasonable to have, as it requires a weaker type of feedback relative to the standard RL one. For this reason, we believe studying it and understanding the gaps between the trajectory feedback RL and standard RL is of importance. The central result of this work is a hybrid optimistic-TS based RL algorithm with a provably bounded $\sqrt{K}$ regret that can be applied when both the reward and transition model are unknown and, thus, needed to be learned. Importantly, the suggested algorithm is computationally tractable, as it requires to solve an empirical MDPs and not a convex maximization problem. 

Regret minimization for standard RL has been extensively studied. Previous algorithms for this scenario can be roughly divided into optimistic algorithms \citep{jaksch2010near,azar2017minimax,jin2018q,dann2019policy,zanette2019tighter,simchowitz2019non,efroni2019tight} and Thompson-Sampling (or Posterior-Sampling) based algorithms \citep{osband2013more,gopalan2015thompson,osband2017posterior,russo2019worst}. Nonetheless, and to the best of our knowledge, we are the first to present a hybrid approach that utilizes both concepts in the same algorithm. Specifically, we combine the optimistic confidence-intervals of UCBVI \citep{azar2017minimax} alongside linear TS for the reward and also take advantage of analysis tools for posterior sampling in RL \citep{russo2019worst}.

In the presence of trajectory-feedback, our algorithms make use of concepts from linear bandits to learn the reward. Specifically, we use both OFUL \citep{abbasi2011improved} and linear TS \citep{agrawal2013thompson,abeille2017linear}, whose regret bounds for $m$-dimension problems after $K$ time-steps with $1$-subgaussian noise are $\Olog\brc*{m\sqrt{K}}$ and $\Olog\brc*{m^{3/2}\sqrt{K}}$, respectively. These bounds directly affect the performance in the RL setting, but the adaptation of OFUL leads to a computationally-intractable algorithm. 
In addition, when there are at most $N$ context, it is possible to achieve a regret bound of $\Olog\brc*{\sqrt{mK\log N}}$ \citep{chu2011contextual}; however, the number of deterministic policies, which are the number of `contexts' for RL with trajectory-feedback, is exponential in $S$, namely, $A^{SH}$. Therefore, such approaches will lead to similar guarantees to OFUL and will also be computationally intractable.

In terms of regret bounds, the minimax regret in the standard RL setting is $\Olog\brc*{\sqrt{SAHT}}$ \citep{osband2016lower,azar2017minimax}, however, for standard RL the reward feedback is much stronger than for RL with trajectory feedback. For linear bandits with $\sqrt{H}$-subgaussian noise, the minimax performance bounds are $\Olog\brc*{m\sqrt{HK}}$ \citep{dani2008stochastic}. Specifically, in RL we set $m=SA$, which leads to $\Olog\brc*{SA\sqrt{HK}}$. Nonetheless, for RL with trajectory feedback and known model, the context space is the average occupancy measures $d_\pi$, which is heavily-structured. It is an open question whether the minimax regret bound remains $\Olog\brc*{SA\sqrt{HK}}$ for RL with trajectory feedback, when the transition model is known, or whether it can be improved. Moreover, when the model is unknown, our algorithm enjoys a regret of $\Olog\br*{S^2A^{3/2}H^{{3/2}}\sqrt{K}}$ when $H\le SA$. A factor of $\sqrt{SA}$ is a direct result of the TS-approach, that was required to make to algorithm tractable, and an additional $\sqrt{S}$ appears when the model is unknown. Moreover, extending OFUL to the case of unknown model and following a similar analysis to Theorem~\ref{theorem: performance ucbvi ts} would still yield this extra $\sqrt{S}$ factor (and would result in a computationally hard algorithm), in comparison to when we know the model.  It is an open question whether this additional factor can also be improved.

Finally, we believe that this work paves the way to many interesting future research directions, notably, studying RL with additional, more realistic, feedback models of the reward. Furthermore, we believe that the results can be adapted to cases where the feedback is a more complex mapping from state-actions into trajectory-reward, and, specifically, a noisy generalized linear model (GLM) of the trajectory \citep{filippi2010parametric,abeille2017linear,kveton2020randomized}. In this case, even though the reward function is not Markovian, our approach should allow deriving regret bounds. More generally, this can be viewed as a form of reward-shaping with theoretical guarantees, which is, in general, an open question.

\section*{Acknowledgments}
This work was partially funded by the Israel Science Foundation under ISF grant number 2199/20. YE is partially supported by the Viterbi scholarship, Technion. NM is partially supported by the Gutwirth Scholarship.

\bibliography{citations.bib}

\clearpage
\appendix
\onecolumn
\section{Nomenclature}
\begin{enumerate}
    \item[] $l_k =\sqrt{\frac{1}{4}SAH\log\br*{\frac{1+kH^2/\lambda}{\delta/10}}} + \sqrt{\lambda SA} \stackrel{(\lambda=H)}{=}\sqrt{\frac{1}{4}SAH\log\br*{\frac{1+kH}{\delta/10}}} + \sqrt{HSA} $
    \item[] $v_k=\sqrt{9SAH\log\frac{kH^2}{\delta/10}}$
    \item[] $c=2\sqrt{2\pi e}$
    \item[] $g_k = l_{k-1} + (2c+1)v_k \br*{\sqrt{SA} + \sqrt{16\log k}} = \Ocal\br*{SA\sqrt{H\log\frac{kH}{\delta}\log k}}.$
    \item[] $b^{pv}_k(s,a) = \sqrt{\frac{H^2\log \frac{40SAH^2k^3}{\delta}}{n_{k}(s,a)\vee 1}}$
    \item[] $\hat{r}_{k}$ -- the Least Square estimator based on samples from $k$ episodes (Equation \eqref{eq: ls estimator for r}).
    \item[] $\hat{d}_k(s,a)=\sum_{h=1}^H \indicator{s=s_h^k,a = a_h^k} $ -- the empirical state-action frequency at the $k^{th}$ episode.
\end{enumerate}

\section{The Expected Elliptical Potential Lemma}\label{supp: expected elliptical potential lemma}

We prove a more general result than Lemma~\ref{lemma: elliptical potential lemma with side information paper} that will be of use for all the algorithms analyzed in this work. Specifically, this more general lemma can be applied for the TS based algorithms, where the policy $\pi_k$ depends on a random noise term drawn at the beginning of the $k^{th}$ episode. Thus, $\pi_k$ for the TS based algorithms is not measurable w.r.t. $F_k$, since the noise is not $F_k$ measurable and $\pi_k$ depends on this noise. 

\begin{restatable}[Expected Elliptical Potential Lemma]{lemma-rst}{ExpectedPotentialLemma} \label{lemma: elliptical potential lemma with side information}
    Let $\brc*{F_{k}^d}_{k=1}^\infty$ be a filtration such that for any $k$ $F_{k}\subseteq F_{k}^d$. Assume that $d_{\pi_k} = \Ex{\hat{d}_k \mid F_{k-1}^d}$. Then, for all $\lambda>0$, it holds that
    \begin{align*}
        &\sum_{k=0}^{K}\Ex{ \norm{d_{\pi_k}}_{A_{k-1}^{-1}}  | F_{k-1}} \leq 4\sqrt{\frac{H^2}{\lambda}K\log\br*{\frac{2K}{\delta}}} + \sqrt{2\frac{H^2}{\lambda}KSA\log\br*{\lambda + \frac{KH^2}{SA}}}\enspace,
    \end{align*}
    uniformly for all $K>0$, with probability greater than $1-\delta.$
\end{restatable} 
    
\begin{proof}
    Applying Jensen's inequality, we get
  \begin{align*}
        &\Ex{ \norm*{d_{\pi_k}}_{A_{k-1}^{-1}}  | F_{k-1}} = \Ex{ \norm*{\Ex{\hat{d}_{k} | F_{k-1}^d}}_{A_{k-1}^{-1}}  \Big\vert F_{k-1}}\\
        &\leq \Ex{ \Ex{\norm{\hat{d}_{k} }_{A_{k-1}^{-1}}\vert F_{k-1}^d}  \Big\vert F_{k-1}} \tag{Jensen's Inequality. Norm is convex.}\\
        &=\Ex{ \norm{\hat{d}_{k} }_{A_{k-1}^{-1}}|  | F_{k-1}}. \tag{Tower Property}
    \end{align*}
  Adding and subtracting the random variable $\norm{\hat{d}_{k} }_{A_{k-1}^{-1}}$ we get
    \begin{align*}
        &(i) \leq \sum_{k=1}^{K}\Ex{ \norm{\hat{d}_{k} }_{A_{k-1}^{-1}}|  | F_{k-1}} =\underbrace{\sum_{k=1}^{K}\Ex{ \norm{\hat{d}_{k} }_{A_{k-1}^{-1}}|  | F_{k-1}} - \norm{\hat{d}_{k} }_{A_{k-1}^{-1}}}_{(a)} + \underbrace{\sum_{k=1}^{K} \norm{\hat{d}_{k} }_{A_{k-1}^{-1}}}_{(b)}.
    \end{align*}
    It is evident that $(a)$ is the sum over a martingale difference sequence. Also, notice that
    $$\norm{\hat{d}_{k} }_{A_{k-1}^{-1}}
    \leq \frac{1}{\sqrt{\lambda}} \norm{\hat{d}_{k}}_2
    \leq \frac{1}{\sqrt{\lambda}} \norm{\hat{d}_{k}}_1
    =\frac{H}{\sqrt{\lambda}}\enspace,$$
    and since $\norm{\hat{d}_{k} }_{A_{k-1}^{-1}}\ge0$, the elements of the martingale differences sequence are bounded by 
    $$\abs*{\Ex{ \norm{\hat{d}_{k} }_{A_{k-1}^{-1}}|  | F_{k-1}} - \norm{\hat{d}_{k} }_{A_{k-1}^{-1}}} 
    \le 2\frac{H}{\sqrt{\lambda}}\enspace.$$
    Thus, applying Azuma-Hoeffding's inequality, for any fixed $K>0$, w.p. at least $1-\delta'$
    \begin{align*}
        (a) \leq \sqrt{8\frac{H^2}{\lambda}K\log\br*{\frac{1}{\delta'}}}.
    \end{align*}
    Setting $\delta' \gets \delta/(2K^2)$ and taking the union bound on all $K>0$ we have that the inequality
    \begin{align*}
    (a) \leq \sqrt{8\frac{H^2}{\lambda}K\log\br*{\frac{2K^2}{\delta}}} \le 4\sqrt{\frac{H^2}{\lambda}K\log\br*{\frac{2K}{\delta}}}
    \end{align*}
    does not holds for some $K>0$ with probability smaller than than $\sum_{K=1}^\infty \frac{\delta}{2K^2}\leq \delta$. Thus, it holds with probability greater than~$1-\delta$.

    Term $(b)$ can be bounded by applying the elliptical potential lemma (Lemma~\ref{lemma: eliptical potential lemma abbasi}) as follows. 
    \begin{align*}
        (b) &\leq \sqrt{K}\sqrt{\sum_{k=0}^{K}\norm{\hat{d}_{k} }_{A_{k-1}^{-1}}^2} \tag{Jensen's inequality}\\
      &=\sqrt{\frac{H^2}{\lambda}}\sqrt{K}\sqrt{\sum_{k=0}^{K} \frac{\lambda}{H^2}\norm{\hat{d}_{k} }_{A_{k-1}^{-1}}^2}\\
      &=\sqrt{\frac{H^2}{\lambda}}\sqrt{K}\sqrt{\sum_{k=0}^{K} \min\br*{\frac{\lambda}{H^2}\norm{\hat{d}_{k} }_{A_{k-1}^{-1}}^2,1}} \tag{$\norm{\hat{d}_{k} }_{A_{k-1}^{-1}}^2\leq \frac{1}{\lambda}\norm{\hat{d}_{k}}^2\leq \frac{H^2}{\lambda}$}\\
        &\leq \sqrt{\frac{H^2}{\lambda}} \sqrt{2KSA\log\br*{\lambda + \frac{KH^2}{SA}}}. \tag{Lemma~\ref{lemma: eliptical potential lemma abbasi}}
    \end{align*}
    
    Combining the bounds on $(a)$, $(b)$, we conclude the proof of the lemma.
    
\end{proof}

\clearpage
\section{OFUL for RL with Trajectory Feedback} \label{supp: OFUL RL per trajectory feedback}

We start by proving the performance bound when the transition model is known.
\OFULrlPer*

\begin{proof}
We define the good event $\mathbb{G}$ as the event that for all $k>0$,
\begin{align*}
    \norm{\hat{r}_k - r}_{A_k} \leq \sqrt{\frac{1}{4}SAH\log\br*{\frac{1+kH^2/\lambda}{\delta/10}}} + \sqrt{\lambda SA}\eqdef l_k,
\end{align*}
and the event that 
\begin{align*}
        \sum_{k=0}^{K}\Ex{ \norm{d_{\pi_k}}_{A_{k-1}^{-1}}  | F_{k-1}} 
        \leq 4\sqrt{\frac{H^2}{\lambda}K\log\br*{\frac{20K}{\delta}}} + \sqrt{2\frac{H^2}{\lambda}KSA\log\br*{\lambda + \frac{KH^2}{SA}}},
    \end{align*}
for all $K>0$.

By Proposition~\ref{proposition: concentration of reward} the first event holds with probability greater than $1-\frac{\delta}{10}$. 
By Lemma~\ref{lemma: elliptical potential lemma with side information} with $F_k^d=F_k$, the second event holds with probability greater than $1-\frac{\delta}{10}$. Taking the union bound establishes that $\Pr\brc*{\mathbb{G}}\geq 1 - \frac{\delta}{5}\ge 1-\delta$.

Now, let $\mathcal{C}_k \eqdef\brc*{\tilde r : \norm{\tilde r - \hat{r}_k}_{A_k}\leq l_k}.$  Conditioning on $\mathbb{G}$, it holds that $r\in \mathcal{C}_k$ for all $k>0$. Thus, 
\begin{align}
    d_{\pi_k}^T \hat{r}_{k-1} + l_{k-1}\norm{d_{\pi_k}}_{A_{k-1}^{-1}} = \max_{\pi}\br*{d_{\pi}^T \hat{r}_{k-1} + l_{k-1}\norm{d_{\pi}}_{A_{k-1}^{-1}}} = \max_{\pi}\max_{\tilde r\in \mathcal{C}_{k-1}} d_{\pi}^T \tilde r \geq d_{\pi^*}^T r, \label{eq: oful known model optimism}
\end{align}
i.e., the algorithm is optimistic. We initially bound the regret following similar analysis to~\citep{abbasi2011improved} as follows.
\begin{align}
    &\Regret(K) = \sum_{k=1}^K d_{\pi^*}^T r - d_{\pi_k}^T r \nonumber \\
    & \leq \sum_{k=1}^K d_{\pi_k}^T \hat{r}_{k-1} + l_{k-1}\norm{d_{\pi_k}}_{A_{k-1}^{-1}} - d_{\pi_k}^T r \tag{Eq.~\eqref{eq: oful known model optimism}} \nonumber\\ 
    &=\sum_{k=1}^K d_{\pi_k}^T( \hat{r}_{k-1} -r)+ l_{k-1}\norm{d_{\pi_k}}_{A_{k-1}^{-1}} \nonumber\\
    &\leq \sum_{k=1}^K \norm{d_{\pi_k}}_{A_{k-1}^{-1}}\norm{ \hat{r}_{k-1} -r}_{A_{k-1}} + l_{k-1}\norm{d_{\pi_k}}_{A_{k-1}^{-1}} \nonumber\\
    &\leq 2l_{K}\sum_{k=1}^K \norm{d_{\pi_k}}_{A_{k-1}^{-1}}, \label{eq: oful known model rel 1}
\end{align}
where the last relation holds conditioning on $\mathbb{G}$ and that $l_K\geq l_k$ for all $k\leq K$.

Setting $\lambda = H$ and observing that conditioning on the good event it holds that 
\begin{align*}
    \sum_{k=0}^{K}\Ex{ \norm{d_{\pi_k}}_{A_{k-1}^{-1}}  | F_{k-1}} \leq \Ocal\br*{\sqrt{HKSA\log\br*{\frac{HK}{\delta}}}},
\end{align*}

we conclude that
\begin{align*}
    \Regret(K) \leq \Ocal\br*{SAH \sqrt{K} \log\br*{\frac{KH}{\delta}}}.
\end{align*}
\end{proof}

This analysis

\clearpage
\section{Thompson Sampling for RL with Trajectory Feedback and Known Model}
\label{supp: proof of ts with known model}
In this section, we prove Theorem~\ref{theorem: performance ts known model}.

\subsection{The Good Event} \label{supp: known model ts good events}
We now specify the good event $\mathbb{G}$. We establish the performance of our algorithm conditioning on the good event. In the following, we show the good event occurs with high probability. Define the following set of events:
\begin{align*}
    &E^r(k) = \brc*{\forall d\in \R^{SA}: |d^T(\hat{r}_k -  r)|\leq l_k \norm{d}_{A_k^{-1}}}\\
    &E^d(K) = \brc*{\sum_{k=1}^{K}\Ex{ \norm{d_{\pi_k}}_{A_{k-1}^{-1}}  | F_{k-1}} \leq 4\sqrt{\frac{H^2}{\lambda}K\log\br*{\frac{20K}{\delta}}} + \sqrt{2\frac{H^2}{\lambda}KSA\log\br*{\lambda + \frac{KH^2}{SA}}} }
\end{align*}
We also define the events in which all former events hold uniformly as $E^{r}= \cap_{k\geq0} E^{r}(k)$ and $E^{d}= \cap_{K\geq1} E^{d}(K)$ and let $\mathbb{G}=  E^{r} \cup E^{d}$.

\begin{restatable}{lemma-rst}{goodeEventsTSKnown} \label{lemma: good events all ts known model}
Let the good event be $\mathbb{G}=  E^{r} \cup E^{d}$. Then, $\Pr\brc*{\mathbb{G}}\geq 1 - \frac{\delta}{2}$.
\end{restatable}
\begin{proof}
The event $\mathbb{E}^r$ holds uniformly for all $k\ge0$, with probability greater than $1-\frac{\delta}{10}$, by Proposition~\ref{proposition: concentration of reward}  combined with Cauchy-Schwarz inequality:
\begin{align*}
    d^T(\hat{r}_k -  r) 
    =\br*{A_k^{-1/2}d}^T\br*{A_k^{1/2}(\hat{r}_k -  r)}
    \le \norm*{A_k^{-1/2}d}_2 \norm*{A_k^{1/2}(\hat{r}_k -  r)}_2
    =\norm*{d}_{A_k^{-1}} \norm*{\hat{r}_k -  r}_{A_k} 
    \le l_k\norm*{d}_{A_k^{-1}}\enspace.
\end{align*}
The event $E^{d}$ holds uniformly by Lemma~\ref{lemma: elliptical potential lemma with side information} with probability greater than $1-\frac{\delta}{10}$, using the filtration $F_k^d = F_k \cup \brc*{\xi_s}_{s=1}^{k+1}$. Importantly, notice that $\pi_k$ is $F_{k-1}^d$-measurable and  
$$ \Ex{\hat{d}_k \mid F_{k-1}^d} = \Ex{\hat{d}_k \mid \pi_k} = d_{\pi_k}.$$
Applying the union, we get $\Pr\brc*{\mathbb{G}} \ge1 - \frac{\delta}{5} \ge 1 - \frac{\delta}{2}$. 
\end{proof}

\subsection{Optimism with Fixed Probability}
We start by stating three lemmas that will be essential to our analysis, both for known and unknown model. The proof of the lemmas can be found in Appendix~\ref{appendix:TS properties}. The first result analyzes the concentration of the TS noise around zero:
\begin{restatable}[Concentration of Thompson Sampling Noise]{lemma-rst}{concentrationTS} \label{lemma: concentration of TS noise}
Let 
    \begin{align*}
        &E^{\xi}(k) =\brc*{ \forall d\in \R^{SA}:\ \abs*{d^T\xi_k} \leq v_k \br*{\sqrt{SA} + \sqrt{16\log k}}\norm{d}_{A_{k-1}^{-1}}}\enspace,
    \end{align*}
where  $\xi_k\sim \mathcal{N}(0,v_k^2 A_{k-1}^{-1})$. Then, for any $k>0$ it holds that $\Pr\br*{E^{\xi}(k)| F_{k-1}} \geq 1-\frac{1}{k^4}.$ Moreover, for any random variable $X\in\R^{SA}$, it holds that 
\begin{align*}
        &\Ex{\abs*{X^T\xi_k}\vert F_{k-1}} \le v_k \br*{\sqrt{SA} + \sqrt{16\log k}}\Ex{\norm{X}_{A_{k-1}^{-1}}\vert F_{k-1}} + \frac{v_k \sqrt{SA}}{k^2}\sqrt{\Ex{\norm*{X}_{A_{k-1}^{-1}}^2\vert F_{k-1}}}\enspace.
    \end{align*}
\end{restatable}
The proof of the high probability bound is similar to~\citep{agrawal2013thompson}, and the bound on the conditional expectation is an extension that we required for our analysis. The next result shows that by perturbing the LS estimator with Gaussian noise, we get an effective optimism with fixed probability. It follows standard analysis of~\citep{agrawal2013thompson}.
\begin{restatable}[Optimism with Fixed Probability]{lemma-rst}{optimismTSKnownModel}
\label{lemma: optimism ts known model} Let $\tilde{r}_k=\hat{r}_{k-1}+\xi_k$ and assume that $\lambda\in\brs*{1,H}$. Then, for any $k>1$, any filtration $F_{k-1}$ such that $E^r(k-1)$ is true and any model $P'$ that is $F_{k-1}$-measurable, it holds that
\begin{align*}
  \Pr\brc*{ d_{\pi^*}(P')^T\tilde{r}_{k} > d_{\pi^*}(P')^Tr  \mid F_{k-1} } \geq \frac{1}{c},
\end{align*}
for $c=2\sqrt{2\pi e}$. Specifically, for $\pi_k\in\arg\max_{\pi'}d_{\pi'}(P')^T\tilde{r}_{k}$ it also holds that 
\begin{align*}
  \Pr\brc*{ d_{\pi_k}(P')^T\tilde{r}_{k} > d_{\pi^*}(P')^Tr  \mid F_{k-1} } \geq \frac{1}{c}.
\end{align*}
\end{restatable}
Notice that since the model is known to the learner, we can fix $P'=P$, but the general statement will be of use when the model is unknown. Finally, we require the following technical result:
\begin{restatable}{lemma-rst}{symmetrizationTS}\label{lemma: symmetrization bound}
Let $\xi_k,\xi'_k\in\R^{SA}$ be i.i.d. random variables given $F_{k-1}$. Also, let $x_{k-1}\in\R^{SA}$ be some $F_{k-1}$-measurable random variable and $P'$ be an $F_{k-1}$-measurable transition kernel. Finally, let $\tilde{\pi}\in\arg\max_{\pi'} d_{\pi'}(P')^T\br*{x_{k-1}+\xi_k}$. Then, 
\begin{align*}
    &\Ex{ \br*{d_{\tilde\pi}(P')^T\br*{x_{k-1} +\xi_k} -\Ex{d_{\tilde\pi}(P')^T\br*{x_{k-1}+\xi_k} | F_{k-1}}}^+ | F_{k-1}} 
    \le \Ex{\abs*{d_{\tilde\pi}(P')^T\xi_k} + \abs*{d_{\tilde\pi}(P')^T\xi_k'}  | F_{k-1}}\enspace,
\end{align*}
where for any $y\in\R$, $\br*{y}^+=\max\brc*{y,0}$.
\end{restatable}

Using these results, we now prove a variation of Lemma 6 from~\citep{russo2019worst}. We find this technique more easily generalizable (relatively to the analysis in e.g.,~\citealt{agrawal2013thompson}) to the next section in which, we extend this approach to the case a model is not given.

\begin{lemma}[Consequences of Optimism]\label{lemma: gap bound by on policy errors ts known model}
Assume that $\lambda\in\brs*{1,H}$. Also, conditioned on $F_{k-1}$, let $\xi_k'$ be an independent copy of $\xi_k$. Then, for any $k>1$ and any filtration $F_{k-1}$ such that $E^r(k-1)$ occurs, it holds that
\begin{align*}
    \Ex{d_{\pi^*}^Tr - d_{\pi_k}^Tr| F_{k-1}} 
    &\le (2c+1)v_k \br*{\sqrt{SA} + \sqrt{16\log k}}\Ex{\norm{d_{\pi_k}}_{A_{k-1}^{-1}}\vert F_{k-1}} + \frac{(2c+1)v_k H\sqrt{SA}}{\sqrt{\lambda}k^2}\\
    &\quad+\Ex{d_{\pi_k}^T\hat{r}_{k-1} -d_{\pi_k}^Tr| F_{k-1}}.
\end{align*}
\end{lemma}
\begin{proof}
We use the following decomposition,
\begin{align}
    \Ex{d_{\pi^*}^Tr - d_{\pi_k}^Tr| F_{k-1}} = \underbrace{\Ex{d_{\pi^*}^Tr - d_{\pi_k}^T\tilde{r}_{k}| F_{k-1}}}_{(i)} +\Ex{d_{\pi_k}^T\tilde{r}_{k} -d_{\pi_k}^Tr| F_{k-1}}. \label{eq: optimism conse rel 1}
\end{align}
We start by bounding $(i)$ and show that
\begin{align}
    (i) \leq c \Ex{\br*{ d_{\pi_k}^T\tilde{r}_{k} - \Ex{d_{\pi_k}^T\tilde{r}_{k}\mid F_{k-1}}}^+| F_{k-1}}. \label{eq: optimism conse rel 2}
\end{align}

If $(i)=d_{\pi^*}^Tr -\Ex{ d_{\pi_k}^T\tilde{r}_{k}| F_{k-1}}<0$ the inequality trivially holds. Otherwise, let $a \equiv d_{\pi^*}^Tr -\Ex{ d_{\pi_k}^T\tilde{r}_{k}| F_{k-1}}\geq 0$. Then,
\begin{align*}
    &\Ex{ \br*{d_{\pi_k}^T\tilde{r}_{k} -\Ex{d_{\pi_k}^T\tilde{r}_{k} | F_{k-1}}}^+ | F_{k-1}} \\
    &\qquad\qquad\geq a \cdot\Pr\br*{ d_{\pi_k}^T\tilde{r}_{k} -\Ex{d_{\pi_k}^T\tilde{r}_{k} | F_{k-1}} \geq a | F_{k-1}} \tag{Markov's inequality}\\
    &\qquad\qquad= \br*{d_{\pi^*}^Tr -\Ex{ d_{\pi_k}^T\tilde{r}_{k}| F_{k-1}}} \cdot\Pr\br*{ d_{\pi_k}^T\tilde{r}_{k}    \geq d_{\pi^*}^Tr | F_{k-1}}\\
    &\qquad\qquad\geq \br*{d_{\pi^*}^Tr -\Ex{ d_{\pi_k}^T\tilde{r}_{k}| F_{k-1}}}\cdot\frac{1}{c} \tag{Lemma~\ref{lemma: optimism ts known model}},
\end{align*}
which implies that $(i)\leq c\Ex{ \br*{d_{\pi_k}^T\tilde{r}_{k} -\Ex{d_{\pi_k}^T\tilde{r}_{k} | F_{k-1}}}^+ | F_{k-1}}$. 
Next, we apply Lemma~\ref{lemma: symmetrization bound} with $x_{k-1}=\hat{r}_{k-1}$ and $P'=P$ (and, therefore, $\tilde\pi=\pi_k$), which yields
\begin{align*}
    &(i)\leq  c\br*{\Ex{\abs*{d_{\pi_k}^T\xi_k}| F_{k-1}} + \Ex{\abs*{d_{\pi_k}^T\xi_k'}  | F_{k-1}}}.
\end{align*}
Substituting the bound on $(i)$ in~\eqref{eq: optimism conse rel 1} and using the definition $\tilde r_k = \hat r_{k-1}+\xi_k$ we get
\begin{align*}
     \Ex{d_{\pi^*}^Tr - d_{\pi_k}^Tr| F_{k-1}} 
     &\leq  c\br*{ \Ex{\abs*{d_{\pi_k}^T\xi_k} | F_{k-1}} + \Ex{\abs*{d_{\pi_k}^T\xi_k'}  | F_{k-1}}}+\Ex{d_{\pi_k}^T\br*{\hat{r}_{k-1} +\xi_k} -d_{\pi_k}^Tr| F_{k-1}} \\
     & \le (c+1)\Ex{\abs*{d_{\pi_k}^T\xi_k} | F_{k-1}} + c\Ex{\abs*{d_{\pi_k}^T\xi_k'}  | F_{k-1}}+\Ex{d_{\pi_k}^T\hat{r}_{k-1} -d_{\pi_k}^Tr| F_{k-1}}.
\end{align*}
Finally, we apply Lemma~\ref{lemma: concentration of TS noise} on the two noise terms, with $X=d_{\pi_k}$. To this end, notice that 
\begin{align*}
    \norm*{d_{\pi_k}}_{A_{k-1}^{-1}} 
    \le \frac{\norm*{d_{\pi_k}}_2}{\sqrt{\lambda}}
    \le \frac{\norm*{d_{\pi_k}}_1}{\sqrt{\lambda}}
    = \frac{H}{\sqrt{\lambda}}
\end{align*}
Then, applying the lemma yields:
\begin{align*}
    \Ex{d_{\pi^*}^Tr - d_{\pi_k}^Tr| F_{k-1}} 
    &\le (2c+1)v_k \br*{\sqrt{SA} + \sqrt{16\log k}}\Ex{\norm{d_{\pi_k}}_{A_{k-1}^{-1}}\vert F_{k-1}} + \frac{(2c+1)v_k H\sqrt{SA}}{\sqrt{\lambda}k^2}\\
    &\quad+\Ex{d_{\pi_k}^T\hat{r}_{k-1} -d_{\pi_k}^Tr| F_{k-1}}.
\end{align*}

\end{proof}

\subsection{Proof of Theorem~\ref{theorem: performance ts known model}}

We are now ready to prove the following result.

\theoremPerformanceTSUCRL*

    \begin{proof}

We start be conditioning on the good event, which occurs with probability greater than $1-\frac{\delta}{2}$. Conditioned on the good event,
\begin{align}
    \Regret(K) &= \sum_{k=1}^K d_{\pi^*}^Tr - d_{\pi_k}^Tr \nonumber \\
    & =  \sum_{k=1}^K  \Ex{\br*{d_{\pi^*}^Tr - d_{\pi_k}^Tr}  | F_{k-1}} +\sum_{k=1}^K d_{\pi^*}^Tr - d_{\pi_k}^Tr - \Ex{\br*{d_{\pi^*}^Tr - d_{\pi_k}^Tr}  | F_{k-1}}\nonumber\\
    & =\underbrace{ \sum_{k=1}^K  \Ex{\br*{d_{\pi^*}^Tr - d_{\pi_k}^Tr}  | F_{k-1}} }_{(i)}+\underbrace{\sum_{k=1}^K  \Ex{d_{\pi_k}^Tr  | F_{k-1}} -  d_{\pi_k}^Tr}_{(ii)} \label{eq: theorem ucrl ts central term known model}
\end{align}
The first term $(i)$ is bounded in Lemma~\ref{lemma: conditional gap bound ts known model} by
\begin{align}
    &(i) \leq \Ocal\br*{g_K\sqrt{HSAK\log\br*{\frac{KH}{\delta}}} + SAH\sqrt{\log\br*{\frac{kH^2}{\delta}}}}, \label{eq: bound on 1 ucrl ts thorem known model}
\end{align}
where $g_k=l_{k-1}+v_k\br*{\sqrt{SA} + \sqrt{16\log k}}.$ The second term, $(ii)$, is the sum over a martingale difference sequence with random variables bounded in $[0,H]$. Applying Azuma-Hoeffding's inequality for probabilities $\frac{\delta}{4K^2}$ and taking the union bound on all $K>0$, we get the following relation holds with probability greater than $1-\frac{\delta}{2}$.
\begin{align}
    (ii)\leq \Ocal\br*{H\sqrt{K\log\br*{{\frac{K}{\delta}}}}}.\label{eq: bound on 2 ucrl ts thorem known model}
\end{align}
Bounding~\eqref{eq: theorem ucrl ts central term known model} by the bounds in \eqref{eq: bound on 1 ucrl ts thorem known model} and \eqref{eq: bound on 2 ucrl ts thorem known model}  concludes the proof since 
$$g_K\leq \Ocal\br*{SA\sqrt{H \log(K)\log\br*{\frac{KH}{\delta}}}}\enspace.$$
The bound holds for all $K>0$ w.p. at least $1-\delta$, using the union bound on the good event and Equation \eqref{eq: bound on 2 ucrl ts thorem known model} (each w.p. $1-\delta/2$).
\end{proof}

\begin{lemma}[Conditional Gap Bound] \label{lemma: conditional gap bound ts known model} Conditioning on the good event and setting $\lambda=H$, the following bound holds.
\begin{align*}
    \sum_{k=1}^K\Ex{(d_{\pi^*}^Tr - d_{\pi_k}^Tr) | F_{k-1}} \leq \Ocal\br*{g_K\sqrt{HSAK\log\br*{\frac{KH}{\delta}}} + SAH\sqrt{\log\br*{\frac{kH^2}{\delta}}} }.
\end{align*}
where $g_k=l_k+(2c+1)v_k\br*{\sqrt{SA} + \sqrt{16\log k}}.$
\end{lemma}
\begin{proof}
    Conditioning on the good event, and by Lemma~\ref{lemma: gap bound by on policy errors ts known model} with $\lambda=H$, we get 
\begin{align*}
    &\sum_{k=1}^K\Ex{(d_{\pi^*}^Tr - d_{\pi_k}^Tr) | F_{k-1}} \\
    &\leq \sum_{k=1}^K\br*{(2c+1)v_k \br*{\sqrt{SA} + \sqrt{16\log k}}\Ex{\norm{d_{\pi_k}}_{A_{k-1}^{-1}}\vert F_{k-1}} + \frac{(2c+1)v_k \sqrt{SAH}}{k^2}+\Ex{d_{\pi_k}^T\hat{r}_{k-1} -d_{\pi_k}^Tr| F_{k-1}}}\\
    & \overset{(1)} \le (2c+1)v_K \br*{\sqrt{SA} + \sqrt{16\log K}} \sum_{k=1}^K\Ex{\norm{d_{\pi_k}}_{A_{k-1}^{-1}}\vert F_{k-1}} + \sum_{k=1}^K\Ex{\norm{d_{\pi_k}}_{A_{k-1}^{-1}} \norm{r-\hat{r}_{k-1}}_{A_{k-1}} | F_{k-1}}\\
    &\quad+  2(2c+1)v_K \sqrt{SAH}\\
     & \overset{(2)} \le \br*{(2c+1)v_K \br*{\sqrt{SA} + \sqrt{16\log K}} + l_{K-1}} \sum_{k=1}^K\Ex{\norm{d_{\pi_k}}_{A_{k-1}^{-1}}\vert F_{k-1}} +  2(2c+1)v_K \sqrt{SAH}\\
    & = g_K\sum_{k=1}^K\Ex{\norm{d_{\pi_k}}_{A_{k-1}^{-1}}\vert F_{k-1}} +  2(2c+1)v_K \sqrt{SAH} \\
    & \overset{(3)} \le g_K\br*{4\sqrt{HK\log\br*{\frac{20K}{\delta}}} + \sqrt{2HKSA\log\br*{\lambda + \frac{KH^2}{SA}}}} +  2(2c+1)v_K \sqrt{SAH} 
    \\
    & =\Ocal\br*{g_K\sqrt{HSAK\log\br*{\frac{KH}{\delta}}} + SAH\sqrt{\log\br*{\frac{kH^2}{\delta}}} }.
\end{align*}
$(1)$ is since $v_k\le v_K$ for any $k\le K$ and $\sum_{k=1}^K\frac{1}{k^2}\le 2$. Next, $(2)$ holds under the good event (and specifically, $E^r$), and since $l_{k-1}\le l_{K-1}$ for any $k\le K$.
where the last relation holds conditioning on $\mathbb{G}$ (and, specifically, $E^r$) and using $l_K\geq l_k$ for any $k\in[K]$. Finally, $(3)$ holds by $E^d$ with $\lambda=H$ and substitution of $v_K$.
\end{proof}

\clearpage
\section{Thompson Sampling for RL with Trajectory Feedback and Unknown Model}\label{supp: ucbvi-ts RL with per trajectory feedback}

\subsection{The Good Event} \label{supp: good events unknown}
We now specify the good event $\mathbb{G}$. We establish the performance of our algorithm conditioning on the good event. In the following, we show the good event occurs with probability greater than $1-\delta/2$. Define the following set of events
\begin{align*}
    &E^r(k) = \brc*{\forall d\in \R^{SA}: |d^T(\hat{r}_k -  r)|\leq l_k \norm{d}_{A_k^{-1}}}\\
    &E^p(k) = \brc*{\forall s\in S, a\in A:\ \norm*{P\br*{\cdot|s,a} - \bar{P}_k\br*{\cdot|s,a}}_1 \le \sqrt{\frac{4S\log\frac{40SAHk^3}{\delta}}{n_{k-1}(s,a)\vee 1}}}\\
    &E^{pv}(k)=\brc*{\forall s,a,h:\ \abs*{\br*{\bar{P}_{k}(\cdot \mid s,a)-P(\cdot \mid s,a)}^T V_{h+1}^*} \leq \sqrt{\frac{H^2\log \frac{40SAH^2k^3}{\delta}}{n_{k}(s,a)\vee 1}}\eqdef b^{pv}_k(s,a)}\\
    &E^N(K) = \brc*{\sum_{k=1}^K \Ex{\sum_{h=1}^H \frac{1}{\sqrt{n_{k-1}(s^{k}_h,a^{k}_h)}} \mid F_{k-1}} \leq 16H^2\log\br*{\frac{20K^2}{\delta}} +4SAH +2\sqrt{2}\sqrt{SAHK\log HK}}\\
    &E^d(K) = \brc*{\sum_{k=0}^{K}\Ex{ \norm{d_{\pi_k}}_{A_{k-1}^{-1}}  | F_{k-1}} \leq 4\sqrt{\frac{H^2}{\lambda}K\log\br*{\frac{20K}{\delta}}} + \sqrt{2\frac{H^2}{\lambda}KSA\log\br*{\lambda + \frac{KH^2}{SA}}} }.
\end{align*}

Furthermore, we define the events in which all the former events hold uniformly, namely, $E^{r}= \cap_{k\geq0} E^{r}(k)$, $E^p= \cap_{k\geq0} E^p(k)$, $E^{pv}= \cap_{k\geq0} E^{pv}(k)$, $E^N = \cap_{K\geq1} E^N(K)$ and $E^{d}= \cap_{K\geq1} E^{d}(K)$. 

\begin{restatable}{lemma-rst}{goodeEventsTSUnknown} \label{lemma: good events all ucbvi-ts}
Let the good event be $\mathbb{G}= E^r \cap E^p \cap E^{pv} \cap E^N \cap E^d$. Then, $\Pr\brc*{\mathbb{G}}\geq 1 - \delta/2$. 
\end{restatable}
\begin{proof}
    We analyze the probability each of the events does not hold and upper bound this probability by $\delta/10$. Taking the union bound then concludes the proof. 
    
    \paragraph{The event $E^r$.} The event holds uniformly for all $k\ge0$, with probability greater than $1-\frac{\delta}{10}$, by Proposition~\ref{proposition: concentration of reward}  combined with Cauchy-Schwarz inequality:
    \begin{align*}
        d^T(\hat{r}_k -  r) 
        =\br*{A_k^{-1/2}d}^T\br*{A_k^{1/2}(\hat{r}_k -  r)}
        \le \norm*{A_k^{-1/2}d}_2 \norm*{A_k^{1/2}(\hat{r}_k -  r)}_2
        =\norm*{d}_{A_k^{-1}} \norm*{\hat{r}_k -  r}_{A_k} 
        \le l_k\norm*{d}_{A_k^{-1}}\enspace.
    \end{align*}
    
    \paragraph{The event $E^p$.} Note that the event trivially holds if $n_k(s,a)=0$, so we assume w.l.o.g. that $n_k(s,a)\ge1$. Next, for any fixed $k$, we apply the following concentration inequality for L1-norm \citep{weissman2003inequalities}:
    \begin{lemma}
    Let $X_1,\dots,X_m$ be i.i.d random values over $\brc*{1,\dots,a}$ such that $\Pr\brc*{X_n=i}=p_i$, and define $\hat{p}_m(i)=\frac{1}{m}\sum_{n=1}^m\indicator{X_n=i}$. Then, for all $\delta,\in\br*{0,1}$,
    \begin{align*}
    \Pr\brc*{\norm*{\hat{p}_m-p}_1\ge \sqrt{\frac{2a\log\frac{2}{\delta,}}{m}}}\le \delta,
    \end{align*}
    \end{lemma}
    Next, we fix $\delta'=\frac{\delta}{20SAHk^3}$. Then, noting that $n_k(s,a)\le kH$ and taking the union bounds over all possible values of $s,a,n_k(s,a)$, we get that $\Pr\brc*{E^p(k)}\ge1-\frac{\delta}{20k^2}$. Finally, taking the union bound over all possible values of $k$ and recalling that $\sum_{k=1}^\infty \frac{1}{k^2}\le 2$, we get that $\Pr\brc*{E^p}\ge1-\delta/10$, which concludes the proof.

    \paragraph{The event $E^{pv}$.} Notice that the event trivially holds if $n_k(s,a)=0$, since $V_{h+1}^*(s')\in\brs*{0,H}$ and $P,\bar{P}_k$ are in the $S$-dimensional simplex for any $s,a$; thus, and w.l.o.g., we assume that $n_k(s,a)\ge1$. Next, we fix $s,a,h$ and $\delta'\in\br*{0,1}$ and let $s'_1,\dots,s'_m$ be $m$ i.i.d samples from $P(\cdot\vert s,a)$. Finally, we define $X_i=V_{h+1}^*(s'_i)\in\brs*{0,H}$ for $i\in\brs*{m}$. Specifically, notice that $\Ex{X_i}=P(\cdot\vert s,a)^TV_{h+1}^*$ and we can write $\bar{P}_k(\cdot\vert s,a)^TV_{h+1}^*=\frac{1}{n_k(s,a)}\sum_{i=1}^{n_k(s,a)}X_i$. Then, By Hoeffding's inequality, w.p. $1-\delta'$
    \begin{align*}
        \Pr\brc*{ \abs*{\frac{1}{m}\sum_{i=1}^m (X_i - \E\brs*{X_i})} \ge \sqrt{ H^2\log\frac{2}{\delta'} }}\le \delta'\enspace.
    \end{align*}
    Next, notice that for any $k$, $n_k(s,a)\in\brs*{Hk}$. Fixing $m=n_k(s,a)$ and taking the union bound over all possible values yields w.p. $1-\delta'$
    \begin{align*}
        \Pr\brc*{ \abs*{\bar{P}_k(\cdot\vert s,a)^TV_{h+1}^*- P(\cdot\vert s,a)^TV_{h+1}^*} \ge \sqrt{ H^2\log\frac{2Hk}{\delta'} } }\le \delta'\enspace.
    \end{align*}
    Finally, choosing $\delta'=\frac{\delta}{20SAHk^2}$ and taking the union bound over all possible values of $s,a,h$ and $k>0$ leads to $\Pr\brc*{E^{pv}}\ge 1-\frac{\delta}{10}$.

    \paragraph{The event $E^N$.}  By Lemma~\ref{lemma: expected cumulative visitation bound}, for any fixed $K$, it holds with probability greater than $1-\delta'$ that
    \begin{align*}
            \sum_{k=1}^K \Ex{\sum_{h=1}^H \frac{1}{\sqrt{n_{k-1}(s^k_h,a^k_h)\vee 1}} \mid F_{k-1}} \leq 16H^2\log\br*{\frac{1}{\delta'}} +4SAH +2\sqrt{2}\sqrt{SAHK\log HK}.
        \end{align*}
    Setting $\delta' =\frac{\delta}{20K^2}$ and applying the union bound we get that 
    \begin{align*}
            \sum_{k=1}^K \Ex{\sum_{h=1}^H \frac{1}{\sqrt{n_{k-1}(s^k_h,a^k_h)\vee 1}} \mid F_{k-1}} \leq 16H^2\log\br*{\frac{20K^2}{\delta}} +4SAH +2\sqrt{2}\sqrt{SAHK\log HK},
        \end{align*}
        for all $K>0$ w.p. at least $1-\delta/10$.

    \paragraph{The event $E^d$.} This event holds with probability greater than $1-\delta/10$ by Lemma~\ref{lemma: elliptical potential lemma with side information} with $\lambda=H$, using the filtration $F_k^d = F_k \cup \brc*{\xi_s}_{s=1}^{k+1}$ as in Lemma~\ref{lemma: good events all ts known model}.
    
\end{proof}

\subsection{Optimism with Fixed Probability}
We remind the reader the $\tilde{r}^b_{k} \eqdef \hat{r}_{k-1} + b_{k-1}^{pv} + \xi_k$ where $\hat{r}_{k-1}$ is the LS-estimator, $b_{k-1}^{pv}$ is the bonus at the beginning of the $k^{th}$ episode, and $\xi_k$ is the randomly drawn noise at the $k^{th}$ episode. The following result which establishes an optimism with fixed probability holds.
\begin{restatable}[Optimism with Fixed Probability]{lemma-rst}{optimismTS}
\label{lemma: optimism ucrl} 
    Assume that $\lambda\in\brs*{1,H}$. Then, for any $k>1$ and any filtration $F_{k-1}$ such that $E^r(k-1)$ and $E^{pv}(k-1)$ are true,
    \begin{align*}
      \Pr\brc*{ d_{\pi_k}(\bar{P})^T\tilde{r}^b_{k} > d_{\pi^*}^Tr  \mid F_{k-1} } \geq \frac{1}{c},
    \end{align*}
    for $c=2\sqrt{2\pi e}$.
\end{restatable}
\begin{proof}
    First, by definition of the update rule, we have
    \begin{align}
      &d_{\pi_k}(\bar P_{k-1})^T\tilde{r}_{k}^b -  d_{\pi^*}^T r = \max_\pi d_{\pi}(\bar P_{k-1})^T\tilde{r}_{k}^b -  d_{\pi^*}^T r\geq  d_{\pi^*}(\bar P_{k-1})^T\tilde r_{k}^b -  d_{\pi^*}^T r. \label{eq: optimsm bonus relation 3}
    \end{align}
    Next, we establish that
    \begin{align}
        d_{\pi^*}(\bar P_{k-1})^T\tilde r_{k}^b -  d_{\pi^*}^T r \geq d_{\pi^*}(\bar P_{k-1})^T\tilde r_{k} - d_{\pi^*}(\bar P_{k-1})^Tr, \label{eq: optimism first step}
    \end{align}
    where $\tilde r_{k} = \hat{r}_{k-1} + \xi_k$.
    By the value difference lemma (Lemma~\ref{lemma: value difference}), and given $F_{k-1}$, it holds that
    \begin{align*}
        d_{\pi^*}&(\bar P_{k-1})^T\tilde r_{k}^b -  d_{\pi^*}^T r \\
        &= \Ex{\sum_{h=1}^H (\tilde{r}_{k}^b-r)(s_h,a_h) + (\bar P_{k-1}(\cdot| s_h,a_h) -P(\cdot| s_h,a_h))^T V^*_{h+1}  | s_1,\bar P_{k-1},\pi^*} \tag{Lemma~\ref{lemma: value difference}}\\
        &= \Ex{\sum_{h=1}^H (\tilde r_{k}-r)(s_h,a_h) + b^{pv}_{k-1}(s_h,a_h)+ (\bar P_{k-1}(\cdot| s_h,a_h) -P(\cdot| s_h,a_h))^T V^*_{h+1}  | s_1,\bar P_{k-1},\pi^*}\\
        &\geq \Ex{\sum_{h=1}^H (\tilde r_{k}-r)(s_h,a_h) + b^{pv}_{k-1}(s_h,a_h) -   b^{pv}_{k-1}(s_h,a_h)  | s_1,\bar P_{k-1},\pi^*} \tag{$E^{pv}(k-1)$ holds in $F_{k-1}$}\\
        &=\Ex{\sum_{h=1}^H (\tilde r_{k}-r)(s_h,a_h) | s_1,\bar P_{k-1},\pi^*} \\
        &= d_{\pi^*}(\bar P_{k-1})^T\tilde r_{k} - d_{\pi^*}(\bar P_{k-1})^Tr,
    \end{align*}
    which establishes~\eqref{eq: optimism first step}. Finally, since $E^r(k-1)$ is true under $F_{k-1}$ and $\bar P_{k-1}$ is $F_{k-1}$-measurable, we can apply Lemma~\ref{lemma: optimism ts known model}, which leads to the desired result:
    \begin{align*}
      \Pr\brc*{ d_{\pi_k}(\bar P_{k-1})^T\tilde{r}^b_{k} > d_{\pi^*}^Tr  \mid F_{k-1} } 
      &= \Pr\brc*{ d_{\pi_k}(\bar P_{k-1})\tilde{r}^b_{k} -d_{\pi^*}^Tr > 0  \mid F_{k-1} }\\
      &\geq \Pr\brc*{d_{\pi^*}(\bar P_{k-1})^T\tilde r_{k}^b -  d_{\pi^*}^T r>0 \mid F_{k-1} } \tag{By~\eqref{eq: optimsm bonus relation 3}}\\
      &\geq \Pr\brc*{d_{\pi^*}(\bar P_{k-1})^T\tilde r_{k} - d_{\pi^*}(\bar P_{k-1})^Tr>0 \mid F_{k-1} } \tag{By~\eqref{eq: optimism first step}}\\
      &\geq \frac{1}{c} \tag{Lemma~\ref{lemma: optimism ts known model}}.
    \end{align*}
\end{proof}

Next, we generalize Lemma~\ref{lemma: gap bound by on policy errors ts known model} from the case a model is known to the case the model is unknown.
\begin{restatable}{lemma-rst}{optimismTSConsequences}\label{lemma: gap bound by on policy errors ucbvi-ts}
    Assume that $\lambda\in\brs*{1,H}$. Then, for any $k>1$ and any filtration $F_{k-1}$ such that $E^r(k-1)$ and $E^{pv}(k-1)$ are true, it holds that 
    \begin{align*}
        \Ex{d_{\pi^*}^Tr - d_{\pi_k}^Tr | F_{k-1}} 
        &\le (2c+1)v_k \br*{\sqrt{SA} + \sqrt{16\log k}}\Ex{\norm{d_{\pi_k}(\bar{P}_{k-1})}_{A_{k-1}^{-1}}\vert F_{k-1}} + \frac{(2c+1)v_k H\sqrt{SA}}{\sqrt{\lambda}k^2}\\
        &\quad+\Ex{d_{\pi_k}(\bar{P}_{k-1})^T\hat{r}_{k-1}^b -d_{\pi_k}^Tr| F_{k-1}}.
    \end{align*}
    where $\hat{r}_{k-1}^b=\hat{r}_{k-1}+b^{pv}_{k-1}$.
\end{restatable}
\begin{proof}
    We use the following decomposition,
    \begin{align}
        \Ex{d_{\pi^*}^Tr - d_{\pi_k}^Tr| F_{k-1}} = \underbrace{\Ex{d_{\pi^*}^Tr - d_{\pi_k}(\bar{P}_{k-1})^T\tilde{r}^b_{k}| F_{k-1}}}_{(i)} +\Ex{d_{\pi_k}(\bar{P}_{k-1})^T\tilde{r}^b_{k} -d_{\pi_k}^Tr| F_{k-1}}. \label{eq: optimism conse rel 4 ts-ucbvi}
    \end{align}
    We start by bounding $(i)$ and show that
    \begin{align}
        (i) \leq c \Ex{\br*{ d_{\pi_k}(\bar{P}_{k-1})^T\tilde{r}^b_{k} - \Ex{d_{\pi_k}(\bar{P}_{k-1})^T\tilde{r}^b_{k}\mid F_{k-1}}}^+| F_{k-1}}. \label{eq: optimism conse rel 5}
    \end{align}
    
    If $(i)=d_{\pi^*}^Tr -\Ex{ d_{\pi_k}(\bar{P}_{k-1})^T\tilde{r}^b_{k}| F_{k-1}}<0$ the inequality trivially holds. Let $a \equiv d_{\pi^*}^Tr -\Ex{ d_{\pi_k}(\bar{P}_{k-1})^T\tilde{r}^b_{k}| F_{k-1}}\geq 0$. Then,
    \begin{align*}
        &\Ex{ \br*{d_{\pi_k}(\bar{P}_{k-1})^T\tilde{r}^b_{k} -\Ex{d_{\pi_k}(\bar{P}_{k-1})^T\tilde{r}^b_{k} | F_{k-1}}}^+ | F_{k-1}} \\
        &\geq a \Pr\br*{ d_{\pi_k}(\bar{P}_{k-1})^T\tilde{r}^b_{k} -\Ex{d_{\pi_k}(\bar{P}_{k-1})^T\tilde{r}^b_{k} | F_{k-1}} \geq a | F_{k-1}} \tag{Markov's inequality}\\
        &= \br*{d_{\pi^*}^Tr -\Ex{ d_{\pi_k}(\bar{P}_{k-1})^T\tilde{r}^b_{k}| F_{k-1}}} \Pr\br*{ d_{\pi_k}(\bar{P}_{k-1})^T\tilde{r}^b_{k}    \geq d_{\pi^*}^Tr | F_{k-1}}\\
        &\geq \br*{d_{\pi^*}^Tr -\Ex{ d_{\pi_k}(\bar{P}_{k-1})^T\tilde{r}^b_{k}| F_{k-1}}}\frac{1}{c} \tag{Lemma~\ref{lemma: optimism ucrl}},
    \end{align*}
    which implies that $(i)\leq c\Ex{ \br*{d_{\pi_k}(\bar{P}_{k-1})^T\tilde{r}^b_{k} -\Ex{d_{\pi_k}(\bar{P}_{k-1})^T\tilde{r}^b_{k} | F_{k-1}}}^+ | F_{k-1}}.$      
    Next, we apply Lemma~\ref{lemma: symmetrization bound} with $x_{k-1}=\hat{r}_{k-1}^b$ and $P'=\bar{P}_{k-1}$ (and, therefore, $\tilde\pi=\pi_k$), which yields
    \begin{align*}
        &(i)\leq  c\br*{\Ex{\abs*{d_{\pi_k}(\bar{P}_{k-1})^T\xi_k}| F_{k-1}} + \Ex{\abs*{d_{\pi_k}(\bar{P}_{k-1})^T\xi_k'}  | F_{k-1}}}
    \end{align*}
    Substituting the bound on $(i)$ in~\eqref{eq: optimism conse rel 4 ts-ucbvi} and using the definition $\tilde{r}_k^b = \hat r_{k-1}^b+\xi_k$ we get
    \begin{align*}
         \Ex{d_{\pi^*}^Tr - d_{\pi_k}^Tr| F_{k-1}} 
         &\leq  c\br*{ \Ex{\abs*{d_{\pi_k}(\bar{P}_{k-1})^T\xi_k} | F_{k-1}} + \Ex{\abs*{d_{\pi_k}(\bar{P}_{k-1})^T\xi_k'}  | F_{k-1}}} \\
         &\quad+\Ex{d_{\pi_k}(\bar{P}_{k-1})^T\br*{\hat{r}_{k}^b +\xi_k} -d_{\pi_k}^Tr| F_{k-1}} \\
         &\le (c+1)\Ex{\abs*{d_{\pi_k}(\bar{P}_{k-1})^T\xi_k} | F_{k-1}} + c\Ex{\abs*{d_{\pi_k}(\bar{P}_{k-1})^T\xi_k'} | F_{k-1}} \\
         &\quad + \Ex{d_{\pi_k}(\bar{P}_{k-1})^T\hat{r}_{k}^b -d_{\pi_k}^Tr| F_{k-1}}
    \end{align*}
    Finally, we apply Lemma~\ref{lemma: concentration of TS noise} on the two noise terms, with $X=d_{\pi_k}$. To this end, notice that 
    \begin{align*}
        \norm*{d_{\pi_k}(\bar{P}_{k-1})}_{A_{k-1}^{-1}} 
        \le \frac{\norm*{d_{\pi_k}(\bar{P}_{k-1})}_2}{\sqrt{\lambda}}
        \le \frac{\norm*{d_{\pi_k}(\bar{P}_{k-1})}_1}{\sqrt{\lambda}}
        = \frac{H}{\sqrt{\lambda}}
    \end{align*}
    Then, applying the lemma yields:
    \begin{align*}
        \Ex{d_{\pi^*}^Tr - d_{\pi_k}^Tr| F_{k-1}} 
        &\le (2c+1)v_k \br*{\sqrt{SA} + \sqrt{16\log k}}\Ex{\norm{d_{\pi_k}(\bar{P}_{k-1})}_{A_{k-1}^{-1}}\vert F_{k-1}} + \frac{(2c+1)v_k H\sqrt{SA}}{\sqrt{\lambda}k^2}\\
        &\quad+\Ex{d_{\pi_k}^T\hat{r}_{k-1}^b -d_{\pi_k}^Tr| F_{k-1}}.
    \end{align*}
\end{proof}
    
\subsection{Proof of Theorem~\ref{theorem: performance ucbvi ts}}\label{supp: proof of ucbvi-ts}
We are now ready to prove the theorem.
\theoremPerformanceTSUCBVI*
\begin{proof}

We start be conditioning on the good event, which occurs with probability greater than $1-\frac{\delta}{2}$ (Lemma~\ref{lemma: good events all ucbvi-ts}). Conditioning on the good event,
\begin{align}
    \Regret(K) &= \sum_{k=1}^K d_{\pi^*}^Tr - d_{\pi_k}^Tr \nonumber\\
    &=  \sum_{k=1}^K \Ex{\br*{d_{\pi^*}^Tr - d_{\pi_k}^Tr} | F_{k-1}} + \sum_{k=1}^K d_{\pi^*}^Tr - d_{\pi_k}^Tr - \Ex{\br*{d_{\pi^*}^Tr - d_{\pi_k}^Tr} | F_{k-1}} \nonumber \\
    & = \underbrace{\sum_{k=1}^K \Ex{\br*{d_{\pi^*}^Tr - d_{\pi_k}^Tr} | F_{k-1}}}_{(i)} + \underbrace{\sum_{k=1}^K \Ex{d_{\pi_k}^Tr | F_{k-1}} - d_{\pi_k}^Tr}_{(ii)}. \label{eq: theorem ucrl ts central term}
\end{align}

The first term $(i)$ is bounded in Lemma~\ref{lemma: conditional gap bound ts unknown model} by
\begin{align}
    &(i) \leq  \Ocal\br*{H^2\sqrt{S}(SA + H)^2\log\br*{\frac{SAHK}{\delta}}^2\sqrt{\log K} + SH(SA+H)\sqrt{AHK\log K} \log\br*{\frac{SAHK}{\delta}}^{\frac{3}{2}} }. \label{eq: bound on 2 ucrl ts thorem}
\end{align}

The second term $(ii)$ is a sum over bounded martingale difference terms, and can be bounded exactly as in the proof of Theorem~\ref{theorem: performance ts known model}, using Azuma-Hoeffding inequality (see Section~\ref{supp: proof of ts with known model}). Specifically, with probability greater than $1-\frac{\delta}{2}$, uniformly for all $K>0$, we have
\begin{align*}
    (ii)\leq \Ocal\br*{H\sqrt{K\log\br*{\frac{K}{\delta}}}},
\end{align*}
Combining the bounds on $(i)$ and $(ii)$ concludes the proof. 
\end{proof}

\begin{lemma}[Conditional Gap Bound Unknown Model] \label{lemma: conditional gap bound ts unknown model} Conditioning on the good event and setting $\lambda=H$, the following bound holds.
\begin{align*}
    &\sum_{k=1}^K\Ex{(d_{\pi^*}^Tr - d_{\pi_k}^Tr) | F_{k-1}} \\
    &\qquad\leq 
    \Ocal\br*{H^2\sqrt{S}(SA + H)^2\log\br*{\frac{SAHK}{\delta}}^2\sqrt{\log K} + SH(SA+H)\sqrt{AHK\log K} \log\br*{\frac{SAHK}{\delta}}^{\frac{3}{2}} }
\end{align*}
\end{lemma}
\begin{proof}
    We start by following similar steps to the proof of Theorem~\ref{theorem: performance ts known model}, when the transition model is known. Conditioning on the good event and by Lemma~\ref{lemma: gap bound by on policy errors ucbvi-ts} with $\lambda=H$, we get 
\begin{align}
    \sum_{k=1}^K\Ex{d_{\pi^*}^Tr - d_{\pi_k}^Tr | F_{k-1}} 
    &\leq \sum_{k=1}^K(2c+1)v_k \br*{\sqrt{SA} + \sqrt{16\log k}}\Ex{\norm{d_{\pi_k}(\bar{P}_{k-1})}_{A_{k-1}^{-1}}\vert F_{k-1}} + \sum_{k=1}^K\frac{(2c+1)v_k \sqrt{SAH}}{k^2}\nonumber\\
    &\quad+\sum_{k=1}^K\Ex{d_{\pi_k}(\bar{P}_{k-1})^T\hat{r}_{k-1}^b -d_{\pi_k}^Tr| F_{k-1}}\nonumber\\
    & \le (2c+1)v_K \br*{\sqrt{SA} + \sqrt{16\log K}} \sum_{k=1}^K\Ex{\norm{d_{\pi_k}(\bar{P}_{k-1})}_{A_{k-1}^{-1}}\vert F_{k-1}} +  2(2c+1)v_K \sqrt{SAH} \nonumber\\ 
    &\quad+ \sum_{k=1}^K\Ex{d_{\pi_k}(\bar{P}_{k-1})^T\hat{r}_{k-1}^b -d_{\pi_k}^Tr| F_{k-1}}\label{eq: ucbvi-ts relation 1 gap bound}
\end{align}
where the last inequality is since $v_k\le v_K$ for any $k\le K$ and $\sum_{k=1}^K\frac{1}{k^2}\le 2$. Next, we further decompose the last term of \eqref{eq: ucbvi-ts relation 1 gap bound} as follows:
\begin{align*}
    \sum_{k=1}^K&\Ex{d_{\pi_k}(\bar{P}_{k-1})^T\hat{r}_{k-1}^b -d_{\pi_k}^Tr| F_{k-1}}\\
    & = \sum_{k=1}^K\Ex{d_{\pi_k}(\bar{P}_{k-1})^T\br*{\hat{r}_{k-1} -r}| F_{k-1}} + \sum_{k=1}^K\Ex{\br*{d_{\pi_k}(\bar{P}_{k-1})-d_{\pi_k}}^Tr| F_{k-1}} +  \sum_{k=1}^K\Ex{d_{\pi_k}(\bar{P}_{k-1})^Tb_{k-1}^{pv}} \\
    & \le \sum_{k=1}^K\Ex{d_{\pi_k}(\bar{P}_{k-1})^T\br*{\hat{r}_{k-1} -r}| F_{k-1}} + \sum_{k=1}^K\Ex{\norm*{d_{\pi_k} - d_{\pi_k}(\bar{P}_{k-1})}_1\norm{r}_{\infty}| F_{k-1}} 
    +  \sum_{k=1}^K\Ex{d_{\pi_k}(\bar{P}_{k-1})^Tb_{k-1}^{pv}} \tag{H\"older's inequality} \\
    &\le l_{K-1}\sum_{k=1}^K\Ex{\norm*{d_{\pi_k}(\bar{P}_{k-1})}_{A_{k-1}^{-1}}| F_{k-1}} + \sum_{k=1}^K\Ex{\norm*{d_{\pi_k} - d_{\pi_k}(\bar{P}_{k-1})}_1| F_{k-1}} 
    +  \sum_{k=1}^K\Ex{d_{\pi_k}(\bar{P}_{k-1})^Tb_{k-1}^{pv}} 
\end{align*}

where the last inequality holds under $E^r$, since $l_{k-1}\le L_{K-1}$ for all $k\le K$ and $\norm*{r}_{\infty}\le1$. Substituting back into \eqref{eq: ucbvi-ts relation 1 gap bound} and defining $g_k=l_k+(2c+1)v_k\br*{\sqrt{SA} + \sqrt{16\log k}}$, we get 
\begin{align}
    \sum_{k=1}^K\Ex{d_{\pi^*}^Tr - d_{\pi_k}^Tr | F_{k-1}}
    &\le g_K\sum_{k=1}^K\Ex{\norm{d_{\pi_k}(\bar{P}_{k-1})}_{A_{k-1}^{-1}}\vert F_{k-1}} + \sum_{k=1}^K\Ex{d_{\pi_k}(\bar{P}_{k-1})^Tb_{k-1}^{pv}}
    \nonumber\\
    &\quad + \sum_{k=1}^K\Ex{\norm*{d_{\pi_k} - d_{\pi_k}(\bar{P}_{k-1})}_1| F_{k-1}}+  2(2c+1)v_K \sqrt{SAH}   \label{eq: ucbvi-ts relation 2 gap bound}
\end{align}
Before we continue, we focus on simplifying the first two terms of \eqref{eq: ucbvi-ts relation 2 gap bound}. The first term of \eqref{eq: ucbvi-ts relation 2 gap bound} can be bounded by
\begin{align*}
   \sum_{k=1}^K  \Ex{\norm{d_{\pi_k}(\bar{P}_{k-1})}_{A_{k-1}^{-1}} \mid F_{k-1}} 
   &\leq \sum_{k=1}^K  \Ex{\norm{d_{\pi_k}}_{A_{k-1}^{-1}} \mid F_{k-1}} +\sum_{k=1}^K  \Ex{\norm{d_{\pi_k}(\bar{P}_{k-1})- d_{\pi_k}}_{A_{k-1}^{-1}} \mid F_{k-1}} \\
   &\leq \sum_{k=1}^K  \Ex{\norm{d_{\pi_k}}_{A_{k-1}^{-1}} \mid F_{k-1}} +\frac{1}{\sqrt{H}} \sum_{k=1}^K  \Ex{\norm{d_{\pi_k}(\bar{P}_{k-1})- d_{\pi_k}}_{1} \mid F_{k-1}},
\end{align*}
where in the last inequality, we used the fact that $\lambda=H$, and therefore, for any $x\in\R^{SA}$,
$$\norm{x}_{A_{k-1}^{-1}} \le \frac{1}{\sqrt{\lambda}}\norm{x}_2=\frac{1}{\sqrt{H}}\norm{x}_2\le\frac{1}{\sqrt{H}}\norm{x}_1,$$
The second term of \eqref{eq: ucbvi-ts relation 2 gap bound} can be bounded by
\begin{align*}
    \sum_{k=1}^K  &\Ex{d_{\pi_k}(\bar{P}_{k-1})^Tb_{k-1}^{pv} \mid F_{k-1}}  \\
    &\qquad= \sum_{k=1}^K  \Ex{\br*{d_{\pi_k}(\bar{P}_{k-1})-d_{\pi_k}}b_{k-1}^{pv} \mid F_{k-1}} +\sum_{k=1}^K  \Ex{ d_{\pi_k}b_{k-1}^{pv} \mid F_{k-1}}\\
    &\qquad\leq \sum_{k=1}^K  \Ex{\norm{d_{\pi_k}(\bar{P}_{k-1})-d_{\pi_k}}\norm{b_{k-1}^{pv}}_\infty \mid F_{k-1}}+\sum_{k=1}^K  \Ex{ d_{\pi_k}b_{k-1}^{pv} \mid F_{k-1}} \tag{H\"older's inequality}\\
     &\qquad\leq H\sqrt{\log\br*{\frac{40SAHK}{\delta}}}\sum_{k=1}^K  \Ex{\norm{d_{\pi_k}(\bar{P}_{k-1})-d_{\pi_k}}_1 \mid F_{k-1}}\\
    &\qquad\quad+H\sqrt{\log\br*{\frac{40SAHK}{\delta}}} \sum_{k=1}^K  \Ex{ \sum_{h=1}^{H} \frac{1}{\sqrt{n_{k-1}(s_h^k,a_h^k)\vee 1}}\mid F_{k-1}}.
\end{align*}
where we bounded $\norm*{b^{pv}_{k-1}}_\infty \le H\sqrt{\log\br*{\frac{40SAHK}{\delta}}}$.

Finally, substituting both these terms back into \eqref{eq: ucbvi-ts relation 2 gap bound} yields
\begin{align}
    \sum_{k=1}^K&\Ex{(d_{\pi^*}^Tr - d_{\pi_k}^Tr) | F_{k-1}} \nonumber\\
    &\leq \underbrace{g_K \sum_{k=1}^K  \Ex{\norm{d_{\pi_k}}_{A_{k-1}^{-1}} \mid F_{k-1}}}_{(a)} 
    + \underbrace{\br*{\frac{g_K}{\sqrt{H}} + H\sqrt{\frac{1}{2}\log\br*{\frac{40SAHK}{\delta}}}+1} \sum_{k=1}^K  \Ex{\norm{d_{\pi_k}(\bar{P}_{k-1})- d_{\pi_k}}_{1} \mid F_{k-1}}}_{(b)}\nonumber \\
    &\quad + \underbrace{H\sqrt{\frac{1}{2}\log\br*{\frac{40SAHK}{\delta}}} \sum_{k=1}^K  \Ex{ \sum_{h=1}^{H} \frac{1}{\sqrt{n_{k-1}(s_h^k,a_h^k)\vee 1}}\mid F_{k-1}}}_{(c)} + 2(2c+1)v_K \sqrt{SAH}  .\label{eq: ucbvi-ts relation 3 gap bound}
\end{align}

\textbf{Bounding $(a)$:} Conditioning on the good event, and, specifically, when $E^d$ holds, term $(a)$ is bounded by
\begin{align*}
   (a) \leq \Ocal\br*{g_K \sqrt{HKSA\log\br*{\frac{HK}{\delta}}}} \leq \Ocal\br*{ (SA)^{3/2}H\sqrt{K \log(K)}\log\br*{\frac{KH}{\delta}}}
\end{align*}
where the second relation holds since $g_K\leq \Ocal\br*{SA\sqrt{H \log(K)\log\br*{\frac{KH}{\delta}}}}$. Observe that this term also appeared in the case the transition model is given (see Section~\ref{supp: proof of ts with known model}).

\textbf{Bounding $(b)$:} Term $(b)$ originates from the fact the transition model is not the true one and, thus, needs to be learned. Specifically, we use the following bound:
\begin{align*}
    &\sum_{k=1}^K  \Ex{\norm{d_{\pi_k}(\bar{P}_{k-1})- d_{\pi_k}}_{1} \mid F_{k-1}}\\
    &\leq H\sum_{k=1}^K  \Ex{\sum_{h=1}^H \norm{P(\cdot \mid s_h^k,a_h^k) -\bar{P}_{k-1}(\cdot \mid s_h^k,a_h^k)}_1 \mid F_{k-1}} \tag{Lemma~\ref{lemma: l1 occupancy measure difference}}\\
    &\leq \Ocal\br*{H\sqrt{S\log\br*{\frac{SAHK}{\delta}}}} \sum_{k=1}^K  \Ex{\sum_{h=1}^H \frac{1}{\sqrt{n_{k-1}(s_h^k,a_h^k)}} \mid F_{k-1}} \tag{The Event $E^p$ holds}\\
    &\leq  \Ocal\br*{H^2\sqrt{S}(SA + H)\log\br*{\frac{SAHK}{\delta}}^{\frac{3}{2}} + SH\sqrt{AHK} \log\br*{\frac{SAHK}{\delta}} } \tag{Event $E^N$ holds}
\end{align*}
Then, noticing that
\begin{align*}
    \frac{g_K}{\sqrt{H}} + H\sqrt{\frac{1}{2}\log\br*{\frac{40SAHK}{\delta}}} + 1 
    = \Ocal\br*{(SA+H)\sqrt{\log K\log\br*{\frac{SAHK}{\delta}}}},
\end{align*}
we get 
\begin{align*}
   (b)\leq \Ocal\br*{H^2\sqrt{S}(SA + H)^2\log\br*{\frac{SAHK}{\delta}}^2\sqrt{\log K} + SH(SA+H)\sqrt{AHK\log K} \log\br*{\frac{SAHK}{\delta}}^{\frac{3}{2}} }.
\end{align*}

    \textbf{Bounding $(c)$:} Term (c) measures the on policy value of the bonus on the estimates MDP. We bound this term using $E^N$ by
    \begin{align*}
        (c)\leq \Ocal\br*{H^2\br*{SA + H\log\br*{\frac{K}{\delta}}}\sqrt{\log\br*{\frac{SAHK}{\delta}}} + \sqrt{SAHK\log HK \log\br*{\frac{SAHK}{\delta}}}}.
    \end{align*}
    Importantly, notice that $(a)$ and $(c)$ are negligible, when compared to $(b)$. Then, substituting the bounds back into \eqref{eq: ucbvi-ts relation 3 gap bound} concludes the proof.
\end{proof}
 
\clearpage

\section{Analysis of the Rarely-Switching UCBVI-TS}\label{supp: rarely switching ucbvi-ts}

\begin{algorithm}[t]
\caption{Rarely-Switching UCBVI-TS} \label{alg: ucbvi-ts doubling}
\begin{algorithmic}
\STATE {\bf Require:} {\small$\delta\in(0,1)$, $\enspace C>0$, $\enspace\lambda=H$, \\ $\qquad\qquad v_k=\sqrt{9SAH\log\frac{kH^2}{\delta/10}}$, \\
$\qquad\qquad l_k= \sqrt{\frac{1}{4}SAH\log\br*{\frac{1+kH^2/\lambda}{\delta/10}}}+\sqrt{\lambda SA}$, \\ 
$\qquad\qquad b^{pv}_k(s,a) = \sqrt{\frac{H^2\log \frac{40SAH^2k^3}{\delta}}{n_{k}(s,a)\vee 1}}$}
\STATE {\bf Initialize:} $A_0=B_0=\lambda I_{SA}$, $Y_0=Z_0=\mathbf{0}_{SA}$, \\ \qquad\qquad\, Counters $n_{0}(s,a) = 0,\, \forall s,a$.
\FOR{$k=1,...,K$}
    \STATE Calculate $\hat{r}_{k-1}$ via LS estimation~\eqref{eq: ls estimator for r} and $\bar{P}_{k}$ by \eqref{eq:transition estimation}
    \STATE Draw noise $\xi_k\sim \mathcal{N}(0,v_{k}^2A_{k-1}^{-1})$ and define $\tilde{r}_k^b = \hat{r}_{k-1}+\xi_k+b^{pv}_{k-1}$
    \STATE Solve empirical MDP with optimistic-perturbed reward, $\pi_k\in \arg\max_{\pi} d(\bar{P}_{k-1})^T\tilde{r}_k^b$ 
    \STATE Play $\pi_k$, observe $\hat{V}_k$ and $\brc*{(s_h^k,a_h^k)}_{h=1}^H$
    \STATE  Update counters $n_k(s,a)$, $B_{k}= B_{k-1} + \hat{d}_k \hat{d}_k^T $ and $Z_{k} = Z_{k-1} + \hat{d}_k \hat{V}_k$.
    \IF{$\det{B_k}>(1+C)\det{A_{k-1}}$}
    \STATE $A_k=B_k$, $Y_k=Z_k$
    \ELSE
    \STATE $A_k=A_{k-1}$, $Y_k=Y_{k-1}$
    \ENDIF
\ENDFOR
\end{algorithmic}
\end{algorithm}

\switchesBound*
\begin{proof}
    For brevity, denote $N = \sum_{k=1}^K \indicator{A_k \neq A_{k-1}}$, the number of times that $A_k$ was updated before time $K$. Also, denote by $\tau_i$, the time of the $i^{th}$ update of $A_k$, with $\tau_0=0$. Specifically, it implies that $A_K=A_{\tau_N}$. Then, by the update rule, for any $i\in\brs*{N}$, it holds that $\det{A_{\tau_i}} \ge (1+C) \det{A_{\tau_{i-1}}}$. Thus, by a simple induction, we have 
    \begin{align*}
       \det{A_{\tau_N}} 
       \ge (1+C)^N \det{A_{\tau_{0}}}
       =(1+C)^N \det{\lambda I_{SA}}
       = (1+C)^N \lambda^{SA}\enspace.
    \end{align*}
    Also, notice that $\tau_N\le K$ and recall that $\norm*{\hat{d}_k}_2\le \norm*{\hat{d}_k}_1=H$; therefore, by Lemma~\ref{lemma: determinant trace lemma abbasi}, we have 
    \begin{align*}
        \det{A_{\tau_N}} 
        \le \br*{\lambda + \frac{\tau_NH^2}{SA}}^{SA}
        \le \br*{\lambda + \frac{KH^2}{SA}}^{SA}\enspace.
    \end{align*}
    Combining both inequalities and reorganizing leads to the desired result and concludes the proof.
\end{proof}

\theoremSwitchingTSUCBVI*
\begin{proof}
    The proof is very similar to the one of Theorem~\ref{theorem: performance ucbvi ts}. Thus, we follow the outline of the proof of Theorem~\ref{theorem: performance ucbvi ts} and only present the required modifications.
    
    \paragraph{Good events:} We use the good events as defined in Appendix~\ref{supp: good events unknown}. Importantly, under the notations of the modified algorithm, the event $E^d(K)$ is defined as 
    \begin{align*}
        E^d(K) = \brc*{\sum_{k=0}^{K}\Ex{ \norm{d_{\pi_k}}_{B_{k-1}^{-1}}  | F_{k-1}} \leq 4\sqrt{\frac{H^2}{\lambda}K\log\br*{\frac{20K}{\delta}}} + \sqrt{2\frac{H^2}{\lambda}KSA\log\br*{\lambda + \frac{KH^2}{SA}}} }.
    \end{align*}
    Then. the good event holds w.p. greater than $1-\frac{\delta}{2}$ (Lemma~\ref{lemma: good events all ucbvi-ts}). We remind the readers that $B_k = \lambda I_{SA} + \sum_{s=1}^{k} \hat{d}_s \hat{d}_s^T$, i.e., it is the updated Gram matrix (without the rarely-switching mechanism).
    
    \paragraph{Extending the expected elliptical lemma for the rarely-switching case: } In the analysis, the expected elliptical lemma (Lemma~\ref{lemma: elliptical potential lemma with side information}) is used to bound the probability of $E^d$, which is defined w.r.t. $B_k$. However, we require a similar result for $A_k$. Denote by $u_k$,  the last time were $A_k$ was updated. Then, $B_k-A_k = \sum_{s=u_k+1}^k \hat{d}_s\hat{d}_s^T$ is positive semi-definite, i.e., $A_k\preceq B_k$. Since both matrices are positive-definite, this also implies that $A_k^{-1}\succeq B_k^{-1}$. Thus, by Lemma~\ref{lemma: determinant ratio lemma abbasi}, for any $x\in\R^{SA}$,
    \begin{align}
        \norm*{x}_{A_k^{-1}} \le \norm*{x}_{B_k^{-1}} \sqrt{\frac{\det{A_k^{-1}}}{\det{B_k^{-1}}}}
        = \norm*{x}_{B_k^{-1}} \sqrt{\frac{\det{B_k}}{\det{A_k}}} \enspace, \label{eq: norm A-B relation}
    \end{align}
    where the last inequality is since for any $A\succ0$, $\det{A^{-1}}=(\det{A})^{-1}$. Next, notice that by the update rule, if $A_k\neq B_k$, then $A_k$ was not updated at round $k$ and 
    $$\det{B_k}\le (1+C)\det{A_{k-1}}=(1+C)\det{A_k}\enspace.$$
    If $A_k = B_k$, then for $C>0$, we also have $\det{B_k}\le(1+C)\det{A_k}$. Substituting back into \eqref{eq: norm A-B relation}, we get for any $k\ge0$ and any $x\in\R^{SA}$,
    \begin{align*}
        \norm*{x}_{A_k^{-1}} \le \sqrt{1+C}\norm*{x}_{B_k^{-1}} ,
    \end{align*}
    which in turn implies that under $E^d$, we have for any $K>0$,
    \begin{align}
        \sum_{k=0}^{K}\Ex{ \norm{d_{\pi_k}}_{A_{k-1}^{-1}}  | F_{k-1}} &\leq \sqrt{1+C}\sum_{k=0}^{K}\Ex{ \norm{d_{\pi_k}}_{B_{k-1}^{-1}}  | F_{k-1}} \nonumber\\
        &\leq 4\sqrt{\frac{H^2}{\lambda}(1+C)K\log\br*{\frac{20K}{\delta}}} + \sqrt{2\frac{H^2}{\lambda}(1+C)KSA\log\br*{\lambda + \frac{KH^2}{SA}}} \label{eq:switching expected elliptical lemma}.
    \end{align}

    \paragraph{TS auxiliary lemmas:} Notice that Lemmas \ref{lemma: optimism ucrl}  and \ref{lemma: gap bound by on policy errors ucbvi-ts} rely on the reward concentration of the LS estimator, according to $E^r(k-1)$. Specifically, if $u_k$ is the last time where $A_k$ was updated, both lemmas would hold by replacing $E^r(k-1)$ by $E^r(u_{k-1})$, as $u_{k-1}$ is $F_{k-1}$-measurable. The only subtlety is the fact that $l_{k-1}$ is replaced by $l_{u_{k-1}}$, but since $u_{k-1}\le k-1$, it also holds that $l_{u_{k-1}}\le l_{k-1}$, and one can easily verify that it is enough for the results to hold.
    
    \paragraph{Conditional gap bounds (Lemma~\ref{lemma: conditional gap bound ts unknown model}):} Since Lemmas \ref{lemma: optimism ucrl}  and \ref{lemma: gap bound by on policy errors ucbvi-ts} still hold, the same derivation holds up to Equation~\eqref{eq: ucbvi-ts relation 3 gap bound}. Moreover, terms $(b)$ and $(c)$ in this equation do not depend on $A_{k-1}$; therefore, they remain the same. Under $E^d$, we can also bound term $(a)$ by Equation \eqref{eq:switching expected elliptical lemma}, namely,
    \begin{align*}
       (a) \leq \Ocal\br*{g_K \sqrt{(1+C)HKSA\log\br*{\frac{HK}{\delta}}}} \leq \Ocal\br*{ (SA)^{3/2}H\sqrt{(1+C)K \log(K)}\log\br*{\frac{KH}{\delta}}}
   \end{align*}
   where we substituted $\lambda=H$ and $g_K\leq \Ocal\br*{SA\sqrt{H \log(K)\log\br*{\frac{KH}{\delta}}}}$. Combining with the bounds on terms $(b)$ and $(c)$, and recalling that term $(b)$ is the dominant between the terms, we get
    \begin{align}
        \sum_{k=1}^K&\Ex{(d_{\pi^*}^Tr - d_{\pi_k}^Tr) | F_{k-1}} \nonumber\\
        &\leq \Ocal\br*{ (SA)^{3/2}H\sqrt{(1+C)K \log(K)}\log\br*{\frac{KH}{\delta}}} \nonumber\\
        &\quad+ \Ocal\br*{H^2\sqrt{S}(SA + H)^2\log\br*{\frac{SAHK}{\delta}}^2\sqrt{\log K} + SH(SA+H)\sqrt{AHK\log K} \log\br*{\frac{SAHK}{\delta}}^{\frac{3}{2}} }  .\label{eq: switching ucbvi-ts gap bound}
    \end{align}
    
    \paragraph{Bounding the regret of the Rarely-Switching UCBVI-TS:} We follow the proof of Theorem~\ref{theorem: performance ucbvi ts}. Conditioned on the good event, which holds w.p. $1-\frac{\delta}{2}$, we can write 
    \begin{align*}
        \Regret(K) &= \sum_{k=1}^K d_{\pi^*}^Tr - d_{\pi_k}^Tr 
        = \underbrace{\sum_{k=1}^K \Ex{\br*{d_{\pi^*}^Tr - d_{\pi_k}^Tr} | F_{k-1}}}_{(i)} + \underbrace{\sum_{k=1}^K \Ex{d_{\pi_k}^Tr | F_{k-1}} - d_{\pi_k}^Tr}_{(ii)}. 
    \end{align*}
    
    The first term $(i)$ is bounded in Equation \eqref{eq: switching ucbvi-ts gap bound}, while the second term $(ii)$ is bounded w.p. $1-\delta/2$ by 
    \begin{align*}
        (ii)\leq \Ocal\br*{H\sqrt{K\log\br*{\frac{K}{\delta}}}},
    \end{align*} 
    exactly as in the proof of Theorem~\ref{theorem: performance ts known model}. Combining the bounds on $(i)$ and $(ii)$, we get 
    \begin{align*}
        \Regret&(K) 
        \le \Ocal\br*{ (SA)^{3/2}H\sqrt{(1+C)K \log(K)}\log\br*{\frac{KH}{\delta}}} \\
        &\quad+ \Ocal\br*{H^2\sqrt{S}(SA + H)^2\log\br*{\frac{SAHK}{\delta}}^2\sqrt{\log K} + SH(SA+H)\sqrt{AHK\log K} \log\br*{\frac{SAHK}{\delta}}^{\frac{3}{2}} } \\
        & = \Olog\br*{SH(SA+H)\sqrt{AHK} + (SA)^{3/2}H\sqrt{(1+C)K} + H^2\sqrt{S}(SA + H)^2}\enspace,
    \end{align*}
    which concludes the proof
    
\end{proof}
\clearpage

\section{Value Difference Lemmas}
\begin{lemma}[Value difference lemma, e.g.,  \citealt{dann2017unifying}, Lemma E.15] \label{lemma: value difference}
Consider two MDPs $M = (S, A, P, r,H)$ and $M' = (S, A, P', r', H)$. For any policy $\pi$ and any  $s,h$ the following relation holds:
\begin{align*}
        V^\pi_h&(s; M) - V^\pi_h(s;M') \\
    &= \E\left[\sum_{h'=h}^H (r_{h'}(s_{h'},a_{h'}) - r'_{h'}(s_{h'},a_{h'})) + (P - P')(\cdot \mid s_{h'},a_{h'})^\top V^{\pi}_{t+1}(\cdot; M')\mid s_h=s,\pi,P \right].
\end{align*}
\end{lemma}

The following lemma is a known result (e.g~\citeauthor{rosenberg2019online,jin2019learning}). It bounds the $L_1$ norm of the occupancy measure of a fixed policy for different models, $\norm{q^\pi(P_1) - q^\pi(P_2)}_1$. We supply here a new proof for the lemma which is solely based on the value difference lemma. Thus, it shows this bound is a direct consequence of the value difference lemma. 

\begin{lemma}[Total Variation Occupancy Measure Difference] \label{lemma: l1 occupancy measure difference}
    Let $P_1$ and $P_2$ represent two transition probabilities defined on a common state-action space. Let $\pi$ be a fixed policy and let $q^\pi(P_1),q^\pi(P_2) \in \R^{SAH}$ be the occupancy measure of $\pi$ w.r.t. $P_1$ and $P_2$, respectively, and w.r.t. a common fixed initial state $s_1$.   Then, 
    \begin{align*}
       \norm{q^\pi(P_1) - q^\pi(P_2)}_1 &= \sum_{s,a,h} |q^\pi_h(s,a;P_1) - q^\pi_h(s,a;P_2)|\\
       &\leq H\Ex{\sum_{h=1}^H\norm{P_1(\cdot \mid s_h,a_h) - P_2(\cdot \mid s_h,a_h)}_1\mid s_1,\pi,P_2}
    \end{align*}
    Furthermore, for $d_\pi(p)$ such that $d_\pi(s,a;p)=\sum_{h=1}^H q_h^\pi(s,a;p)$, it holds that
    \begin{align*}
       \norm{d_\pi(P_1) - d_\pi(P_2)}_1 
       &\leq H\Ex{\sum_{h=1}^H\norm{P_1(\cdot \mid s_h,a_h) - P_2(\cdot \mid s_h,a_h)}_1\mid s_1,\pi,P_2}
    \end{align*}
\end{lemma}
\begin{proof}
    Fix $s,a,h$. Observe that  
    \begin{align*}
        q^\pi_h(s,a;P_1) = \Ex{ \indicator{s_h=s,a_h=a} \mid s_1,P_1,\pi} = \Ex{\sum_{h'=1}^H r^{s,a,h}_{h'}(s_{h'},a_{h'})\mid s_1,P_1,\pi},
    \end{align*}
    where 
    \begin{align*}
        r^{s,a,h}_{h'}(s',a') =\indicator{s'=s,a'=a,h'=h}.
    \end{align*}
    
    Thus, we can interpret $q^\pi_h(s,a;P_1)-q^\pi_h(s,a;P_2)$ as the difference in values w.r.t. to two MDPs $M_1= (S,A,P_1,r^{s,a,h},H)$, $M_2= (S,A,P_2,r^{s,a,h},H)$ and fixed policy $\pi$. Doing so, we can apply the value difference lemma and get
    \begin{align*}
        \abs*{q^\pi_h(s,a;P_1)-q^\pi_h(s,a;P_2)} 
        &= \abs*{ \sum_{h'}^H\Ex{ \br*{P_1(\cdot \mid s_{h'},a_{h'}) - P_2(\cdot \mid s_{h'},a_{h'})}^T V^\pi_{h'+1}(s,a,h,P_1) \mid s_1, \pi, P_2}}\\
        &\leq \sum_{h'}^H\Ex{ \sum_{s''}\abs*{P_1(s'' \mid s_{h'},a_{h'}) - P_2(s'' \mid s_{h'},a_{h'}) } V^\pi_{h'+1}(s'';s,a,h,P_1) \mid s_1, \pi, P_2}.
    \end{align*}
    where the inequality is by the triangle inequality and since $V\ge0$, Summing on both sides on $s,a,h$ and using linearity of expectation we get
    \begin{align*}
        \sum_{s,a,h} &\abs*{q^\pi_h(s,a;P_1)-q^\pi_h(s,a;P_2)}\leq \sum_{h'}^H\Ex{ \sum_{s''}\abs*{P_1(s'' \mid s_{h'},a_{h'}) - P_2(s'' \mid s_{h'},a_{h'}) } \br*{\sum_{s,a,h}V^\pi_{h'+1}(s'';s,a,h,P_1)} \mid s_1, \pi, P_2}\enspace.
    \end{align*}
    
    Next, Observe that by definition,
    \begin{align*}
        V_{h'+1}(s'';s,a,h,P_1) = \Ex{\indicator{s_{h}=s,a_h=a} \mid s'',\pi,P_1},
    \end{align*}
    for $h'+1\leq h$ and zero otherwise. Using this, we see that for any $s''$ and $h'+1$
    \begin{align*}
        \sum_{s,a,h}V^\pi_{h'+1}(s'';s,a,h,P_1) = \Ex{\sum_{s,a}\sum_{h=h'+1}^H \indicator{s_{h}=s,a_h=a} \mid s_{h'+1}=s'',\pi,P_1}, \leq H
    \end{align*}
    since in a single trajectory a policy can visit at most at $H$ state-action pairs. Plugging this bound back and observing that 
    \begin{align*}
        \sum_{s''}\abs*{P_1(s'' \mid s_{h'},a_{h'}) - P_2(s'' \mid s_{h'},a_{h'}) } = \norm{P_1(\cdot \mid s_{h'},a_{h'}) - P_2(\cdot \mid s_{h'},a_{h'})}_1
    \end{align*}
    concludes the proof of the first inequality. The second inequality is a direct result of the first inequality and the triangle inequality, namely,
    \begin{align*}
       \norm{d_\pi(P_1) - d_\pi(P_2)}_1 
       &= \sum_{s,a}\abs*{\sum_h q_h^\pi(s,a;P_1) - q_h^\pi(s,a;P_2)} \\
       &\le \sum_{s,a,h}\abs*{q_h^\pi(s,a;P_1) - q_h^\pi(s,a;P_2)} \tag{triangle inequality}\\
       &= \norm{q^\pi(P_1) - q^\pi(P_2)}_1  \\
       &\leq H\Ex{\sum_{h=1}^H\norm{P_1(\cdot \mid s_h,a_h) - P_2(\cdot \mid s_h,a_h)}_1\mid s_1,\pi,P_2}\enspace.
    \end{align*}
\end{proof}

\clearpage
\section{Cumulative Visitation Bounds}
The following result follows the analysis of Lemma 10 in~\citep{jin2019learning} and generalizes their result to the stationary MDP setting (by doing so, a factor of $\sqrt{H}$ is shaved off relatively to their bound). This result can also be found in other works such as~\citep{dann2015sample}. However, we find the analysis somewhat simpler and for this reason we supply it here.

Observe that for the TS based algorithms $\pi_k$ is not $F_{k-1}$ measurable since the policy $\pi_k$ depends on a noise drawn at the beginning of the $k^{th}$ episode. For this reason, in the following analysis, we use $\Ex{d_{\pi_k} \mid F_{k-1}}$ instead of $d_{\pi_k}$ as in the analysis of~\citep{jin2019learning} which was applied for an optimistic algorithm (there, $\pi_k$ is a deterministic function of the history). Using this trick, we also generalize the results of~\citep{jin2019learning} since, for the optimistic case, it holds that $\Ex{d_{\pi_k} \mid F_{k-1}} = d_{\pi_k}$.

\begin{lemma}[Expected Cumulative Visitation Bound] \label{lemma: expected cumulative visitation bound}
    Let $\brc*{F_{k}}_{k=1}^K$ be the natural filtration. Then, with probability greater than $1-\delta$ it holds that
    \begin{align*}
        \sum_{k=1}^K \Ex{\sum_{h=1}^H \frac{1}{\sqrt{n_{k-1}(s^k_h,a^k_h)\vee 1}} \mid F_{k-1}} &\leq 16H^2\log\br*{\frac{1}{\delta}} +4SAH +2\sqrt{2}\sqrt{SAHK\log HK}\\
        &=\Ocal\br*{H\br*{SA + H\log\br*{\frac{1}{\delta}}} + \sqrt{SAHK\log HK}}.
    \end{align*}
\end{lemma}
\begin{proof}
    First, notice that under our model, given the policy $\pi_k$, the trajectory $\brc*{s_h^k,a_h^k}_{h\in\brs*{H}}$ is independent of $F_{k-1}$. Therefore, by the tower property, we have for any $s,a$ that
    \begin{align}
        \Ex{ \sum_{h=1}^H \indicator{s^k_h=s,a^k_h=a} \mid F_{k-1}}
        &= \Ex{ \Ex{ \sum_{h=1}^H \indicator{s^k_h=s,a^k_h=a}\mid F_{k-1},\pi_k} \mid F_{k-1}} \nonumber\\
        &= \Ex{ \Ex{d_{\pi_k}(s,a)\mid F_{k-1},\pi_k} \mid F_{k-1}} \nonumber\\
        &= \Ex{ d_{\pi_k}(s,a)\mid F_{k-1}}.  \label{eq: conditional indicator to occupancy}
    \end{align}
    Next, we write the l.h.s. of the desired inequality as
    \begin{align*}
        \sum_{k=1}^K\Ex{\sum_{h=1}^H \frac{1}{\sqrt{n_{k-1}(s^k_h,a^k_h)\vee 1}} \mid F_{k-1}} 
        & =\sum_{k=1}^K\sum_{s,a}\Ex{\sum_{h=1}^H \frac{\indicator{s_h^k=s,a_h^k=a}}{\sqrt{n_{k-1}(s,a)\vee 1}} \mid F_{k-1}} 
        = \sum_{k=1}^K\sum_{s,a} \frac{\Ex{d_{\pi_k}(s,a) \mid F_{k-1}}}{\sqrt{n_{k-1}(s,a)\vee 1}},
    \end{align*}
    where the last equality is by \eqref{eq: conditional indicator to occupancy}, and decompose it by
    \begin{align}
        \sum_{k=1}^K\sum_{s,a} &\frac{\Ex{d_{\pi_k}(s,a) \mid F_{k-1}}}{\sqrt{n_{k-1}(s,a)\vee 1}} \nonumber\\
        &= \underbrace{ \sum_{k=1}^K\sum_{s,a} \frac{\Ex{d_{\pi_k}(s,a) \mid F_{k-1}}-\sum_{h=1}^H \indicator{s^k_h=s,a^k_h=a}}{\sqrt{n_{k-1}(s,a)\vee 1}}}_{(i)} 
        +\sum_{k=1}^K\sum_{s,a} \frac{\sum_{h=1}^H \indicator{s^k_h=s,a^k_h=a}}{\sqrt{n_{k-1}(s,a)\vee 1}}. \label{eq:cummulative bound rel decomp}
    \end{align}
    Notably, observe that 
    \begin{align}
        \abs*{\sum_{s,a} \frac{\Ex{d_{\pi_k}(s,a) \mid F_{k-1}}-\sum_{h=1}^H \indicator{s^k_h=s,a^k_h=a}}{\sqrt{n_{k-1}(s,a)\vee 1}}}\leq H. \label{eq: freedman 1 counts}
    \end{align}
    We now apply apply Freedman's Inequality (Lemma~\ref{lemma: freedmans inequality}) to bound $(i)$. Set
    $$
    Y_k = \sum_{s,a} \frac{\Ex{d_{\pi_k}(s,a) \mid F_{k-1}}-\sum_{h=1}^H \indicator{s^k_h=s,a^k_h=a}}{\sqrt{n_{k-1}(s,a)\vee 1}},
    $$
    and see that $|Y_k|\leq H$ by~\eqref{eq: freedman 1 counts} and that it is a martingale difference sequence by~\eqref{eq: conditional indicator to occupancy}. Applying Freedman's Inequality with $\eta=1/8H$ we get
    \begin{align}
        (i)&\leq \sum_{k=1}^K \frac{1}{8H}\Ex{\br*{ \sum_{s,a} \frac{\Ex{d_{\pi_k}(s,a) \mid F_{k-1}}-\sum_{h=1}^H \indicator{s^k_h=s,a^k_h=a}}{\sqrt{n_{k-1}(s,a)\vee 1}}}^2\mid F_{k-1}} + 8H^2\log\br*{\frac{1}{\delta}}. \label{eq: cummulative bound rel 1}
    \end{align}
    Furthermore, we have that
    \begin{align*}
        &\Ex{\br*{ \sum_{s,a} \frac{\Ex{d_{\pi_k}(s,a) \mid F_{k-1}}-\sum_{h=1}^H \indicator{s^k_h=s,a^k_h=a}}{\sqrt{n_{k-1}(s,a)\vee 1}}}^2\mid F_{k-1}}\\
        &\stackrel{(1)}\leq 2\Ex{\br*{ \sum_{s,a} \frac{\sum_{h=1}^H \indicator{s^k_h=s,a^k_h=a}}{\sqrt{n_{k-1}(s,a)\vee 1}}}^2\mid F_{k-1}} + 2\Ex{\br*{ \sum_{s,a} \frac{\Ex{d_{\pi_k}(s,a) \mid F_{k-1}}}{\sqrt{n_{k-1}(s,a)\vee 1}}}^2\mid F_{k-1}} \\
        &\stackrel{(2)}\leq 2H\br*{\Ex{ \sum_{s,a} \frac{\sum_{h=1}^H \indicator{s^k_h=s,a^k_h=a}}{\sqrt{n_{k-1}(s,a)\vee 1}}\mid F_{k-1}} + \sum_{s,a} \frac{\Ex{d_{\pi_k}(s,a) \mid F_{k-1}}}{\sqrt{n_{k-1}(s,a)\vee 1}}}\\
        & \stackrel{(3)}=4H \sum_{s,a} \frac{\Ex{d_{\pi_k}(s,a) \mid F_{k-1}}}{\sqrt{n_{k-1}(s,a)\vee 1}}\enspace.
    \end{align*}
    $(1)$ is since $\forall a,b,:\br*{a+b}^2\le 2(a^2+b^2)$, In $(2)$, notice that for both terms, the argument of the square function $f(x)=x^2$ is in $\brs*{0,H}$; then, we can bound $x^2\le Hx$. Finally, for $(4)$ we use Equation~\eqref{eq: conditional indicator to occupancy}. Plugging this and~\eqref{eq: cummulative bound rel 1} back into \eqref{eq:cummulative bound rel decomp}, we get
    \begin{align*}
        \sum_{k=1}^K\sum_{s,a} \frac{\Ex{d_{\pi_k}(s,a) \mid F_{k-1}}}{\sqrt{n_{k-1}(s,a)\vee 1}} \leq \frac{1}{2}\sum_{k=1}^K\sum_{s,a} \frac{\Ex{d_{\pi_k}(s,a) \mid F_{k-1}}}{\sqrt{n_{k-1}(s,a)\vee 1}}  
        +8H^2\log\br*{\frac{1}{\delta}}
        +\sum_{k=1}^K\sum_{s,a}\frac{ \sum_{h=1}^H \indicator{s^k_h=s,a^k_h=a}}{\sqrt{n_{k-1}(s,a)\vee 1}},
    \end{align*}
    and, thus,
    \begin{align}
        \sum_{s,a}\sum_{k=1}^K \frac{\Ex{d_{\pi_k}(s,a) \mid F_{k-1}}}{\sqrt{n_{k-1}(s,a)\vee 1}} \leq  16H^2\log\br*{\frac{1}{\delta}} + 2\sum_{s,a} \sum_{k=1}^K\frac{ \sum_{h=1}^H\indicator{s^k_h=s,a^k_h=a}}{\sqrt{n_{k-1}(s,a)\vee 1}}.\label{eq: cummulative expected relation almost final}
    \end{align}
    To bound the second term in~\eqref{eq: cummulative expected relation almost final} we fix an $s,a$ pair and bound the following sum.
    \begin{align}
        &\sum_{k=1}^K\frac{ \sum_{h=1}^H\indicator{s^k_h=s,a^k_h=a}}{\sqrt{n_{k-1}(s,a)\vee 1}}\nonumber \\
        &=\sum_{k=1}^K\indicator{n_{k-1}(s,a)\vee 1< H}\frac{ \sum_{h=1}^H\indicator{s^k_h=s,a^k_h=a}}{\sqrt{n_{k-1}(s,a)\vee 1}} + \sum_{k=1}^K\indicator{n_{k-1}(s,a)\vee 1\geq H}\frac{ \sum_{h=1}^H\indicator{s^k_h=s,a^k_h=a}}{\sqrt{n_{k-1}(s,a)\vee 1}}\nonumber\\
        &\stackrel{(1)}\leq 2H + \sum_{k=1}^K\indicator{n_{k-1}(s,a)\vee 1\geq H}\frac{ \sum_{h=1}^H\indicator{s^k_h=s,a^k_h=a}}{\sqrt{n_{k-1}(s,a)\vee 1}}\nonumber\\
        &\stackrel{(2)}\leq 2H + \sqrt{KH}\sqrt{\sum_{k=1}^K\indicator{n_{k-1}(s,a)\vee 1\geq H}\frac{ \sum_{h=1}^H\indicator{s^k_h=s,a^k_h=a}}{n_{k-1}(s,a)\vee 1}}\nonumber\\
        &\stackrel{(3)}\leq 2H +\sqrt{2KH(\log n_{K}(s,a)\vee 1)} \nonumber ,
    \end{align}
    where $(1)$ is since if $\indicator{s^k_h=s,a^k_h=a}=1$, then $n_k(s,a)$ will increase by $1$; therefore, both indicators in the first term can be true only $2H-1\le2H$ times (the extreme case is when $n_{k-1}(s,a)=H-1$ and $s_h^k=s,a_h^k=a$ for all $h\in\brs*{H}$). Next, $(2)$ is by Cauchy-Schwartz inequality and $(3)$ holds by Lemma~\ref{lemma: cumulative visitation bound}. Plugging this bound back into~\eqref{eq: cummulative expected relation almost final} we get
    \begin{align*}
        &\eqref{eq: cummulative expected relation almost final} \leq  16H^2\log\br*{\frac{1}{\delta}} + 2\sum_{s,a}  \br*{2H +\sqrt{2KH\log n_{K}(s,a)\vee 1)}}\\
        &\leq 16H^2\log\br*{\frac{1}{\delta}} +4SAH +2\sqrt{2}\sqrt{SAHK\log\sum_{s,a}n_{K}(s,a)} \tag{Jensen's Inequality}\\
        &=16H^2\log\br*{\frac{1}{\delta}} +4SAH +2\sqrt{2}\sqrt{SAHK\log HK} \tag{$\sum_{s,a}n_{K}(s,a)=HK$}
    \end{align*}
    where we used the fact that $\sqrt{\log(x)}$ is concave for $x\geq 1$. This concludes the proof.
\end{proof}
    
\begin{lemma}[Cumulative Visitation Bound] \label{lemma: cumulative visitation bound}
    For any fixed $(s,a)$ pair, it holds that
    \begin{align*}
        \sum_{k=1}^K \indicator{n_{k-1}(s,a)\geq H} \frac{ \sum_{h=1}^H\indicator{s^k_h=s,a^k_h=a}}{n_{k-1}(s,a)\vee 1}\leq  2\log \br*{n_{K}(s,a)\vee 1}\enspace. 
    \end{align*}
\end{lemma}
\begin{proof}
   The following relations hold for any fixed $s,a$ pair.
    \begin{align*}
        &\sum_k \indicator{n_{k-1}(s,a)\geq H} \frac{ \sum_{h=1}^H\indicator{s^k_h=s,a^k_h=a}}{n_{k-1}(s,a)\vee 1}\\
        &=\sum_k \indicator{n_{k-1}(s,a)\geq H}\frac{ \sum_{h=1}^H\indicator{s^k_h=s,a^k_h=a}}{n_{k}(s,a)} \frac{n_{k}(s,a)}{n_{k-1}(s,a)} \tag{$n_k(s,a)\ge n_{k-1}(s,a)\ge H\ge1$}\\
        &\leq 2\sum_k \indicator{n_{k-1}(s,a)\geq H}\frac{ \sum_{h=1}^H\indicator{s^k_h=s,a^k_h=a}}{n_{k}(s,a)} \tag{$\frac{n_{k}(s,a)}{n_{k-1}(s,a)}\leq \frac{n_{k-1}(s,a)+H}{n_{k-1}(s,a)}\stackrel{n_{k-1}(s,a)\ge H}\leq 2$}\\
         &=2\sum_k \indicator{n_{k-1}(s,a)\geq H}\frac{ n_{k}(s,a) - n_{k-1}(s,a)}{n_{k}(s,a)}\\
         &\leq 2\sum_k \indicator{n_{k-1}(s,a)\geq H} \log\br*{\frac{n_{k}(s,a)}{n_{k-1}(s,a)}} \tag{$\frac{a-b}{a}\leq \log\frac{a}{b}$ for $a\geq b>0$}\\
         &\leq \indicator{n_{K}(s,a)\geq H}\cdot2\log n_{K}(s,a) - 2\log(H)\\
         &\leq 2\log \br*{n_{K}(s,a)\vee 1}.
    \end{align*}
\end{proof}
    
\clearpage
\section{Useful Results}
\label{appendix:useful results}

\begin{theorem}[\citealt{abbasi2011improved}, Theorem 2]\label{theorem: abassi confidence interval}

Let $\brc*{F_k}_{k=0}^\infty$ be a filtration. Let $\brc*{\eta_k}_{k=0}^\infty$ be a real-valued stochastic process such that $\eta_k$ is $F_k$-measurable and $\eta_k$ is conditionally $\sigma$-sub-Gaussian for $\sigma\geq 0$. Let $\brc*{x_k}_{k=0}^\infty$ be an $\R^m$-valued stochastic process s.t. $X_k$ is $F_{k-1}$-measurable and $\norm{x_k}\leq L$. Define $y_k = \inner{x_k,w}+\eta_t$ and assume that $\norm{w}_2\leq R $ and $\lambda>0$. Let
\begin{align*}
    \hat{w}_t = (X_k^TX_k+\lambda I_d)^{-1} X_k^T Y_k,
\end{align*}
where $X_k$ is the matrix whose rows are $x_1^T,..,x_t^T$ and $Y_k = (y_1,..,y_k)^T$. Then, for any $\delta>0$ with probability at least $1-\delta$ for all, $t\geq 0$ $w$ lies in the set
\begin{align*}
    \brc*{w\in \R^m: \norm{\hat{w}_k - w}_{V_k} \leq \sigma\sqrt{m\log{\frac{1+kL^2/\lambda}{\delta}}} + \lambda^{1/2}R}.
\end{align*}
\end{theorem}
The previous theorem can be easily extended to our setting, as stated in Proposition~\ref{proposition: concentration of reward}:
\rewardConcentration*
\begin{proof}
Define the filtration $\tilde{F}_k=\sigma\br*{\hat{d}_{\pi_1},\dots,\hat{d}_{\pi_{k+1}},\eta_1\dots,\eta_k}$, where 
$$\eta_k = \sum_{h=1}^H \br*{R(s_h^k,a_h^k)-r(s_h^k,a_h^k} = \sum_{h=1}^HR(s_h^k,a_h^k) - \hat{d}_{\pi_k}^Tr.$$
Specifically, notice that $\hat{d}_{k}\in\R^{SA}$ is $F_{k-1}$ measurable, $\eta_k$ is $F_k$ measurable and that $\eta_k$ is $\sqrt{H/4}$ sub-Gaussian given $F_{k-1}$, as a (centered) sum of $H$ conditionally independent random variables bounded in $\brs*{0,1}$. Also note that $\norm{r}_2\le \sqrt{SA}$ and $\norm{\hat{d}_{k}}_2\le H$. Then, for $A_k = D_kD_k^T + \lambda I$, we can apply Theorem 2 of \citep{abbasi2011improved} (also restated in Theorem~\ref{theorem: abassi confidence interval}). Specifically, the theorem implies that for any $\delta'>0$ with probability at least $1-\delta'$, it holds that
\begin{align*}
    \forall k\ge0,\, \norm{\hat{r}_k-r}_{A_k} \le \sqrt{\frac{1}{4}SAH\log\br*{\frac{1+kH^2/\lambda}{\delta'}}}+\sqrt{\lambda SA} 
\end{align*}
Applying this results for $\delta'=\frac{\delta}{10}$ concludes the proof.
\end{proof}

\begin{lemma}[Determinant-Trace Inequality, \citep{abbasi2011improved}, Lemma 10] \label{lemma: determinant trace lemma abbasi}
Let $X_1,\dots,X_k\in \R^m$ such that $\norm*{X_s}_2\le L$ for all $s\in\brs*{k}$. If $A_k=\lambda I_m + \sum_{s=1}^k X_sX_s^T$, for some $\lambda>0$, then
\begin{align*}
    \det{A_k} \le \br*{\lambda + \frac{kL^2}{m}}^m\enspace.
\end{align*}
\end{lemma}

\begin{lemma}[Elliptical Potential Lemma, \citep{abbasi2011improved}, Lemma 11] \label{lemma: eliptical potential lemma abbasi}
Let $\brc*{x_k}_{k=1}^\infty$ be a sequence in $\R^m$ and $V_k= V+\sum_{i=1}^k x_i x_i^T$. Assume $\norm{x_k}\leq L$ for all $k$. Then,
\begin{align*}
    \sum_{i=1}^k \min\br*{\norm{x_i}_{V_{i-1}^{-1}}^2,1} \leq 2\log\br*{\frac{\det{V_k}}{\det{V}}}  \leq  2m\log\br*{\frac{\trace{V} + kL^2}{m}} -2\log{\det{V}}.
\end{align*}

Furthermore, if $\lambda_{min}(V)\geq \max(1,L^2)$ then
\begin{align*}
    \sum_{i=1}^t \norm{x_i}_{V_{i-1}^{-1}}^2 \leq 2\log{\frac{\det{V_t}}{\det{V}}}  \leq  2m\log{\frac{\trace{V} + tL^2}{m}}.
\end{align*}
\end{lemma}

\begin{lemma}[\citealt{beygelzimer2011contextual}, Freedman's Inequality] \label{lemma: freedmans inequality}
Let $Y_1,..,Y_K$ be a martingale difference sequence w.r.t. a filtration $\brc*{F_{k}}_{k=1}^K$. Assume $|Y_k|\leq R$ a.s. for all $k$. Then, for any $\delta\in(0,1)$ and $\eta\in [0,1/R]$, with probability greater than $1-\delta$ it holds that
\begin{align*}
    \sum_{k=1}^K Y_k \leq \eta \sum_{k=1}^k \Ex{Y_k^2 \mid F_{k-1}} +\frac{R}{\eta} \log\br*{\frac{1}{\delta}}.
\end{align*}
\end{lemma}

\begin{lemma}[\citep{abbasi2011improved}, Lemma 11] \label{lemma: determinant ratio lemma abbasi}
Let $A$ and $B$ be positive definite such that $A\succeq B$, i.e., $A-B$ is positive semi-definite. Then,
\begin{align*}
    \sup_{x\neq 0} \frac{\norm{x}_A^2}{\norm{x}_B^2} \le \frac{\det{A}}{\det{B}}\enspace.
\end{align*}
\end{lemma}

\clearpage

\section{Properties of Thompson Sampling}
\label{appendix:TS properties}

\concentrationTS*
\begin{proof}
    Let $d\in\R^{SA}$ and notice that given $F_{k-1}$, $A_{k-1}$ is fixed. Thus, we have 
    \begin{align}
        \abs*{d^T \xi_k  }
        &= \abs*{d^T A_{k-1}^{-1/2}\cdot A_{k-1}^{1/2}\xi_k  } \nonumber\\
        &\stackrel{(1)}\le \norm*{A_{k-1}^{-1/2}d}_2\norm*{A_{k-1}^{1/2}\xi_k}_2 \nonumber\\
        & =  v_k\norm*{d}_{A_{k-1}^{-1}}\norm*{\frac{1}{v_k}A_{k-1}^{1/2}\xi_k}_2\nonumber\\
        & \stackrel{(2)}=  v_k\norm*{d}_{A_{k-1}^{-1}}\norm*{\zeta_k}_2\enspace. \label{eq:event theta bound}
    \end{align}
    $(1)$ is due to the Cauchy-Schwartz inequality. In $(2)$ we defined $\zeta_k=\frac{1}{v_k}A_{k-1}^{1/2}\xi_k\in\R^{SA}$, which is a vector with independent standard Gaussian components. Specifically, notice that $\norm*{\zeta_k}_2$ is chi-distributed. Then, we can apply Lemma 1 of \citep{laurent2000adaptive}, which implies that w.p. at least $1-\delta$,
    \begin{align*}
        \norm*{\zeta_k}_2 
        &\le \sqrt{SA + 2\sqrt{SA\log \frac{1}{\delta}} + 2 \log \frac{1}{\delta} }
        = \sqrt{\br*{\sqrt{SA} + \sqrt{\log \frac{1}{\delta}}}^2 + \log \frac{1}{\delta} } 
        \le \sqrt{SA} + \sqrt{\log \frac{1}{\delta}} +\sqrt{\log \frac{1}{\delta}} \\
        & =\sqrt{SA} + 2\sqrt{\log \frac{1}{\delta}}\enspace.
    \end{align*}
    Taking $\delta=\frac{1}{k^4}$ and substituting this bound into \eqref{eq:event theta bound} implies that w.p. at least $1-\frac{1}{k^4}$, 
    \begin{align*}
        \abs*{d^T \xi_k  }
        \le v_k\norm*{d}_{A_{k-1}^{-1}}\br*{\sqrt{SA} + \sqrt{16\log k}}\enspace,
    \end{align*}
    and thus $\Pr\brc*{E^{\xi}(k) | F_{k-1}}\ge 1-\frac{1}{k^4}$.
    
    To prove the second result of the lemma, we decompose the conditional expectation as follows:
    \begin{align}
        \Ex{\abs*{X^T\xi_k}\vert F_{k-1}}
         &= \Ex{\abs*{X^T\xi_k}\indicator{E^{\xi}(k)}\vert F_{k-1}} + \Ex{\abs*{X^T\xi_k}\indicator{\overline{E^{\xi}(k)}}\vert F_{k-1}} .\label{eq:TS noise cond expect}
    \end{align}
    For the first term of \eqref{eq:TS noise cond expect}, notice that the bound in $E^{\xi}(k)$ holds for any $d\in\R^{SA}$; therefore, it also holds for any random variable $X\in\R^{SA}$:
    \begin{align*}
        \Ex{\abs*{X^T\xi_k}\indicator{E^{\xi}(k)}\vert F_{k-1}}
        &\le v_k \br*{\sqrt{SA} + \sqrt{16\log k}}\Ex{\norm{X}_{A_{k-1}^{-1}}\indicator{E^{\xi}(k)}\vert F_{k-1}}\\
        & \le v_k \br*{\sqrt{SA} + \sqrt{16\log k}}\Ex{\norm{X}_{A_{k-1}^{-1}}\vert F_{k-1}}.
    \end{align*}
    For the second term of \eqref{eq:TS noise cond expect}, we apply Cauchy-Schwarz (CS) inequality:
    \begin{align*}
        \Ex{\abs*{X^T\xi_k}\indicator{\overline{E^{\xi}(k)}}\vert F_{k-1}} 
        &\overset{(CS)}\le \sqrt{\Ex{\abs*{X^T\xi_k}^2\vert F_{k-1}} }\sqrt{\Ex{\indicator{\overline{E^{\xi}(k)}}\vert F_{k-1}} } \\
        &= v_k\sqrt{\Ex{\abs*{X^TA_{k-1}^{-1/2}\zeta_k}^2\vert F_{k-1}} }\sqrt{\Pr\brc*{\overline{E^{\xi}(k)}\vert F_{k-1}} }\\
        &\overset{(CS)}\le v_k\sqrt{\Ex{\norm*{X}_{A_k^{-1}}^2\vert F_{k-1}}\Ex{\norm*{\zeta_k}_2^2\vert F_{k-1}} }\sqrt{\Pr\brc*{\overline{E^{\xi}(k)}\vert F_{k-1}} } \\
        & \le \frac{v_k \sqrt{SA}}{k^2}\sqrt{\Ex{\norm*{X}_{A_k^{-1}}^2\vert F_{k-1}}}
    \end{align*}
    where the last inequality is since $\Pr\brc*{\overline{E^{\xi}(k)}\vert F_{k-1}}\le\frac{1}{k^4}$ and $\Ex{\norm*{\zeta_k}_2^2\vert F_{k-1}}=SA$. The proof is concluded by substituting both results back into \eqref{eq:TS noise cond expect}.
\end{proof}

\optimismTSKnownModel*
\begin{proof}
    First, recall that under $E^r(k-1)$, for any fixed $d\in \R^{SA}$, it holds that $\abs*{d^T \hat{r}_{k-1} - d^T r  }\leq l_{k-1} \norm{d}_{A_{k-1}^{-1}}$. Specifically, the relation holds for $d=d_{\pi^*}(P')$. Also recall that conditioned on $F_{k-1}$, $d_{\pi^*}(P')^T\tilde{r}_k$ is a Gaussian random variable with mean $d_{\pi^*}(P')^T\hat{r}_{k-1}$ and standard deviation $v_k\norm*{d_{\pi^*}(P')}_{A_{k-1}^{-1}}$. Then, we can write
    \begin{align*}
        \Pr\brc*{d_{\pi^*}(P')^T\tilde{r}_k > d_{\pi^*}^T(P')r | F_{k-1}} 
        = \Pr\brc*{\frac{d_{\pi^*}(P')^T\tilde{r}_k-d_{\pi^*}(P')^T\hat{r}_{k-1}}{v_k\norm*{d_{\pi^*}(P')}_{A_{k-1}^{-1}}} > \frac{d_{\pi^*}^T(P')r-d_{\pi^*}(P')^T\hat{r}_{k-1}}{v_k\norm*{d_{\pi^*}(P')}_{A_{k-1}^{-1}}}| F_{k-1}} \enspace,
    \end{align*}
    Next, we define $Z_k=\frac{d_{\pi^*}^T(P')r-d_{\pi^*}(P')^T\hat{r}_{k-1}}{v_k\norm*{d_{\pi^*}(P')}_{A_{k-1}^{-1}}}$ and bound $Z_k$ as follows:
    \begin{align*}
        Z_k
        \le \abs*{Z_k} 
        = \frac{\abs*{d_{\pi^*}^T(P')r-d_{\pi^*}(P')^T\hat{r}_{k-1}}}{v_k\norm*{d_{\pi^*}(P')}_{A_{k-1}^{-1}}}
        \stackrel{(1)} \le \frac{l_{k} \norm*{d_{\pi^*}(P')}_{A_{k-1}^{-1}}}{v_k\norm*{d_{\pi^*}(P')}_{A_{k-1}^{-1}}} 
        \stackrel{(2)}= \frac{\sqrt{\frac{1}{4}SAH\log\br*{\frac{1+kH^2/\lambda}{\delta/10}}}+\sqrt{\lambda SA}}{\sqrt{9SAH\log\frac{kH^2}{\delta/10}}}
        \stackrel{(3)}\le 1
    \end{align*}
    where $(1)$ is by the definition of $E^r(k-1)$ and by bounding $l_{k-1}\le l_k$, $(2)$ is a direct substitution of $l_k$ and $v_k$ and $(3)$ holds for any $\lambda\in\brs*{1,H}$ and $k\ge1$. Then, we can write
    \begin{align*}
        \Pr\brc*{d_{\pi^*}(P')^T\tilde{r}_k > d_{\pi^*}^T(P')r | F_{k-1}} 
        &= \Pr\brc*{\frac{d_{\pi^*}(P')^T\tilde{r}_k-d_{\pi^*}(P')^T\hat{r}_{k-1}}{v_k\norm*{d_{\pi^*}(P')}_{A_{k-1}^{-1}}} > Z_k| F_{k-1}} \\
        &\ge\Pr\brc*{\frac{d_{\pi^*}(P')^T\tilde{r}_k-d_{\pi^*}(P')^T\hat{r}_{k-1}}{v_k\norm*{d_{\pi^*}(P')}_{A_{k-1}^{-1}}} > 1| F_{k-1}} \\
        &\ge \frac{1}{2\sqrt{2\pi e}}
        \eqdef\frac{1}{c}\enspace,
    \end{align*}
    where the last inequality is since for $X\sim \mathcal{N}(0,1)$ and any $z>0$, $\Pr\brc*{X> z}\ge \frac{1}{\sqrt{2\pi}}\frac{z}{1+z^2}e^{-z^2/2}$ \citep{borjesson1979simple}. Finally, notice that $d_{\pi_k}(P')^T\tilde{r}_{k} = \max_{\pi'}d_{\pi'}(P')^T\tilde{r}_{k}\ge d_{\pi^*}(P')^T\tilde{r}_{k}$. Thus, the second result of the lemma is a direct consequence of its first result.
    \end{proof}
    
\symmetrizationTS*
\begin{proof}
    First, observe that since $d_{\tilde\pi}(P')^T\br*{x_{k-1}+\xi_k} = \max_{\pi'} d_{\pi'}(P')^T\br*{x_{k-1}+\xi_k}$ and as $\xi_k$, $\xi_k'$ are identically distributed, it also holds that $d_{\tilde\pi}(P')^T\br*{x_{k-1}+\xi_k}$ and $\max_{\pi'} d_{\pi'}(P')^T\br*{x_{k-1}+\xi'_k}$ are identically distributed. Therefore,
    \begin{align}
        &\Ex{ \br*{d_{\tilde\pi}(P')^T\br*{x_{k-1} +\xi_k} -\Ex{d_{\tilde\pi}(P')^T\br*{x_{k-1}+\xi_k} | F_{k-1}}}^+ | F_{k-1}}  \nonumber\\
        &\qquad\qquad= \Ex{ \br*{\max_{\pi'}d_{\pi'}(P')^T\br*{x_{k-1} +\xi_k} -\Ex{\max_{\pi'}d_{\pi'}(P')^T\br*{x_{k-1}+\xi_k'} | F_{k-1}}}^+ | F_{k-1}}, \label{eq: optimism conse rel 3}
    \end{align}
    Next, by definition, it holds that
    \begin{align*}
        &\max_{\pi'}d_{\pi'}(P')^T\br*{x_{k-1} +\xi_k'} \geq  d_{\tilde\pi}(P')^T\br*{x_{k-1} +\xi_k'}.
    \end{align*}
    Notice that conditioned on $F_{k-1}$, $\xi_k$ and $\xi_k'$ are independent, and that $x_{k-1}$ and $P'$ are $F_{k-1}$-measurable. Thus, we can write 
    \begin{align*}
        \Ex{ \max_{\pi'}d_{\pi'}(P')^T\br*{x_{k-1} +\xi_k'} | F_{k-1}} 
        &=\Ex{ \max_{\pi'}d_{\pi'}(P')^T\br*{x_{k-1} +\xi_k'} | F_{k-1},\xi_k} \\
        &\overset{(*)}= \Ex{ \max_{\pi'}d_{\pi'}(P')^T\br*{x_{k-1} +\xi_k'} | F_{k-1},\xi_k,\pi_k} \\
        &\ge \Ex{ d_{\tilde\pi}(P')^T\br*{x_{k-1} +\xi_k'} | F_{k-1},\xi_k,\pi_k} ,
    \end{align*}
    where $(*)$ is since $\pi_k$ is deterministically determined by $F_{k-1}$ and $\xi_k$. Plugging this back into~\eqref{eq: optimism conse rel 3} we get 
    \begin{align*}
        \eqref{eq: optimism conse rel 3} 
        &\leq \Ex{ \br*{d_{\tilde\pi}(P')^T\br*{x_{k-1} +\xi_k} -\Ex{d_{\tilde\pi}(P')^T\br*{x_{k-1} +\xi_k'} | F_{k-1},\xi_k, \pi_k}}^+ | F_{k-1}}\\
        &=\Ex{ \br*{\Ex{d_{\tilde\pi}(P')^T\br*{x_{k-1} +\xi_k} - d_{\tilde\pi}(P')^T\br*{x_{k-1} +\xi_k'} | F_{k-1},\xi_k, \pi_k}}^+ | F_{k-1}}\\
        &=\Ex{ \br*{\Ex{d_{\tilde\pi}(P')^T\xi_k - d_{\tilde\pi}(P')^T\xi_k' | F_{k-1},\xi_k, \pi_k}}^+ | F_{k-1}} \\
        &\leq \Ex{ \abs*{\Ex{d_{\tilde\pi}(P')^T\xi_k - d_{\tilde\pi}(P')^T\xi_k' | F_{k-1},\xi_k, \pi_k}} | F_{k-1}}\tag{$\forall x\in\R: (x)^+\le\abs{x}$}\\
        &\leq \Ex{ \Ex{\abs*{d_{\tilde\pi}(P')^T\xi_k - d_{\tilde\pi}(P')^T\xi_k'} | F_{k-1},\xi_k, \pi_k} | F_{k-1}} \tag{Jensen's Inequality}\\
        &= \Ex{\abs*{d_{\tilde\pi}(P')^T\xi_k - d_{\tilde\pi}(P')^T\xi_k'}  | F_{k-1}} \tag{Tower Property}\\
        &\le \Ex{\abs*{d_{\tilde\pi}(P')^T\xi_k}| F_{k-1}} + \Ex{\abs*{d_{\tilde\pi}(P')^T\xi_k'}  | F_{k-1}} \tag{Triangle Inequality}
    \end{align*}
\end{proof}
\end{document}